\documentclass[english,preprint]{elsarticle}
\usepackage[T1]{fontenc}
\usepackage[latin9]{inputenc}
\usepackage{amsmath}
\usepackage{amsthm}
\usepackage{amssymb}
\usepackage{stmaryrd}
\usepackage{graphicx}

\makeatletter

\providecommand{\tabularnewline}{\\}

\numberwithin{equation}{section}
\numberwithin{figure}{section}
\newenvironment{lyxlist}[1]
{\begin{list}{}
{\settowidth{\labelwidth}{#1}
 \setlength{\leftmargin}{\labelwidth}
 \addtolength{\leftmargin}{\labelsep}
 }}
{\end{list}}
\theoremstyle{plain}
\newtheorem{thm}{\protect\theoremname}
  \theoremstyle{definition}
  \newtheorem{defn}[thm]{\protect\definitionname}
 \theoremstyle{plain}
 \newtheorem*{property*}{\protect\propertyname}
 \theoremstyle{plain}
 \newtheorem{requirement}{\protect\requirementname}
  \theoremstyle{plain}
  \newtheorem{lem}[thm]{\protect\lemmaname}
  \theoremstyle{plain}
  \newtheorem{cor}[thm]{\protect\corollaryname}
 \theoremstyle{plain}
 \newtheorem{property}[thm]{\protect\propertyname}

\date{}

\makeatother

\usepackage{babel}
  \providecommand{\corollaryname}{Corollary}
  \providecommand{\definitionname}{Definition}
  \providecommand{\lemmaname}{Lemma}
 \providecommand{\propertyname}{Property}
 \providecommand{\requirementname}{R\!\!}
\providecommand{\theoremname}{Theorem}

\begin{document}

\begin{frontmatter}{}

\title{From Propositional Logic to Plausible Reasoning: A Uniqueness Theorem}

\tnotetext[t1]{Submitted and currently under review at \emph{International Journal
of Approximate Reasoning}.}

\tnotetext[t2]{\copyright 2017. This manuscript version is made available under
the CC-BY-NC-ND 4.0 license \texttt{http://creativecommons.org/licenses/by-nc-nd/4.0/}.}

\author{Kevin S. Van Horn}

\address{Adobe Systems, 3900 Adobe Way, Lehi, UT 84043, United States}

\ead{vanhorn@adobe.com}
\begin{abstract}
We consider the question of extending propositional logic to a logic
of plausible reasoning, and posit four requirements that any such
extension should satisfy. Each is a requirement that some property
of classical propositional logic be preserved in the extended logic;
as such, the requirements are simpler and less problematic than those
used in Cox's Theorem and its variants. As with Cox's Theorem, our
requirements imply that the extended logic must be isomorphic to (finite-set)
probability theory. We also obtain specific numerical values for the
probabilities, recovering the classical definition of probability
as a \emph{theorem}, with truth assignments that satisfy the premise
playing the role of the ``possible cases.'' 
\end{abstract}

\end{frontmatter}{}

\section{Introduction}

E.\ T.\ Jaynes \cite[p. xxii]{ptlos} proposes the view that probability
theory is the \emph{uniquely determined} extension of classical propositional
logic (CPL) to a ``logic of plausible reasoning'' :
\begin{quote}
Our theme is simply: \emph{probability theory as extended logic}.
\ldots{}the mathematical rules of probability theory are not merely
rules for calculating frequencies of `random variables'; they are
also the unique consistent rules for conducting inference (i.e.\ plausible
reasoning) of any kind\ldots{}
\end{quote}
This view is grounded in the work of Pólya \cite{polya} and Cox \cite{cox},
especially the latter. In this paper we aim to set the notion of probability
theory as the necessary extension of CPL on solid footing.

Our goal is to generalize the logical consequence relation, which
deals only in certitudes, to handle \emph{degrees} of certainty. Whereas
$X\models A$ means that $A$ (the conclusion) is a logical consequence
of $X$ (the premise), we write $A\mid X$ for ``the reasonable credibility
of the proposition $A$ (the query) when the proposition $X$ (the
premise) is known to be true'' (paraphrasing Cox \cite{cox}.) We
call $\left(\cdot\mid\cdot\right)$ the \emph{plausibility function}.\emph{
}If $X\models A$ then $A\mid X$ is some value indicating ``certainly
true,'' if $X\models\neg A$ then $A\mid X$ is some value indicating
``certainly false,'' and otherwise $A\mid X$ is a value indicating
some intermediate level of plausibility. Our task is to determine
what the plausibility function must be, based on logical criteria.

As a generalization of the logical consequence relation, the plausibility
function must depend \emph{only} on its two explicit arguments; the
value it returns must not depend on any additional information that
varies according to the problem domain to which it is applied, nor
according to the intended meanings of the propositional symbols. This
is a \emph{formal logical} theory we are developing, and so any intended
semantics of the propositional symbols must be expressed axiomatically
in the premise.

One might question whether all relevant information for determining
the plausibility of some proposition $A$ can be expressed in propositional
form for inclusion in the premise $X$. Might not our background information
include ``soft'' relationships, mere \emph{propensities} for propositions
to be associated in some way? Although we provide some suggestive
examples, we do not attempt to resolve that question. Instead we ask,
\emph{given} that the background information and intended semantics
are expressed in propositional form and included in the premise, with
no other information available, what can we conclude about the plausibility
function?

We posit four Requirements for the plausibility function. Each of
these requires that some property of the logical consequence relation
be retained in the generalization to a plausibility function. Three
are invariance properties, and the fourth is a requirement to preserve
distinctions in degree of plausibility that already exist within CPL.
These Requirements (discussed in detail later) are the following:
\begin{lyxlist}{R5.}
\item [{R\ref{req:equivalence}.}] If $X$ and $Y$ are logically equivalent,
and $A$ and $B$ are logically equivalent assuming $X$, then $A\mid X=B\mid Y$
(Section \ref{sec:Invariance-from-logical-equiv}.)
\item [{R\ref{req:definitions}.}] We may define a new propositional symbol
without affecting the plausibility of any proposition that does not
mention that symbol. Specifically, if $s$ is a propositional symbol
not appearing in $A$, $X$, or $E$, then $A\mid X=A\mid\nobreak\left(s\leftrightarrow E\right)\wedge X$
(Section \ref{sec:Invariance-under-definition}.)
\item [{R\ref{req:irrelevance}.}] Adding irrelevant information to the
premise does not affect the plausibility of the query. Specifically,
if $Y$ is a satisfiable propositional formula that uses no propositional
symbol occurring in $A$ or $X$, then $A\mid X=A\mid Y\wedge X$
(Section \ref{sec:Invariance-under-addition-of-irrelevant}.)
\item [{R\ref{req:ordering}.}] The implication ordering is preserved:
if $X\models A\rightarrow B$ but not $X\models B\rightarrow A$ then
$A\mid X$ is a plausibility value that is strictly less than $B\mid X$
(Section \ref{sec:Ordering-Properties}.)
\end{lyxlist}
Note that we do \emph{not} assume that plausibility values are real
numbers, nor that they are totally ordered; R\ref{req:ordering} presumes
only that there is some \emph{partial} order on plausibility values.

Given R1\textendash R\ref{req:last}, we prove that plausibilities
are essentially probabilities in disguise. Specifically, we show that
\begin{enumerate}
\item there is an order-preserving isomorphism $P$ between the set of plausibility
values $\mathbb{P}$ and the set of rational probabilities $\mathbb{Q}\cap[0,1]$;
\item $P\left(A\mid X\right)$, the plausibility $A\mid X$ mapped via $P$
to the unit interval, is \emph{necessarily} the ratio of the number
of truth assignments that satisfy both $A$ and $X$ to the number
of truth assignments that satisfy $X$; and
\item hence the usual laws of probability follow as a consequence.
\end{enumerate}
This identifies finite-set probability theory as the uniquely determined
extension of CPL to a logic of plausible reasoning.

The body of this paper is organized as follows:
\begin{itemize}
\item In Section \ref{sec:Relation-to-Prior-Work} we compare this work
to Cox's Theorem and variants, as well as Carnap's system of logical
probability.
\item In Section \ref{sec:Logical-Preliminaries} we review some notions
from CPL, discuss the partial plausibility ordering that \emph{already
exists} within CPL, and discuss the nature of the plausibility function.
\item Our main result is proven in Sections \ref{sec:Invariance-from-logical-equiv},
\ref{sec:Invariance-under-definition}, \ref{sec:Invariance-under-addition-of-irrelevant},
and \ref{sec:Ordering-Properties}, which also introduce the Requirements,
discuss the motivation behind them, and explore some of their consequences.
Along the way we discuss how Carnap's system violates R\ref{req:irrelevance}.
\item In Section \ref{sec:Consistency-of-Requirements} we prove that R1\textendash R\ref{req:last}
are consistent.
\item Section \ref{sec:Discussion} discusses three topics: the connection
of our results to the classical definition of probability, the issue
of non-uniform probabilities, and an initial attempt at extending
our results to infinite domains.
\end{itemize}

\section{Relation to Prior Work\label{sec:Relation-to-Prior-Work}}

To set the context for this paper, clarify our goals, and head off
possible misconceptions, we now review similar prior work and point
out the differences.

\subsection{Cox's Theorem}

R.\ T.\ Cox \cite{cox} proposes a handful of intuitively-appealing,
qualitative requirements for any system of plausible reasoning, and
shows that these requirements imply that any such system is just probability
theory in disguise. Specifically, he shows that there is an order
isomorphism between plausibilities and the unit interval $[0,1]$
such that $A\mid X$, after mapping from plausibilities to $[0,1]$,
respects the laws of probability. 

Over the years Cox's arguments have been refined by others \cite{aczel,ptlos,paris,tribus,vanhorn},
making explicit some requirements that were only implicit in Cox's
original presentation, and replacing some of the requirements with
slightly less demanding assumptions than those used in Cox's original
proof. One version of the requirements \cite{vanhorn} may be summarized
as follows:
\begin{lyxlist}{C5.}
\item [{C1.}] $A\mid X$ is a real number.
\item [{C2.}] $A\mid X=A'\mid X$ whenever $A$ is logically equivalent
to $A'$, and $B\mid X\leq A\mid X$ for any tautology $A$.
\item [{C3.}] There exists a nonincreasing function $S$ such that $\neg A\mid X=S\left(A\mid X\right)$
for all $A$ and satisfiable $X$.
\item [{C4.}] The set of plausibility triples $\left(y_{1},y_{2},y_{3}\right)$
where $y_{1}=A_{1}\mid X$, $y_{2}=A_{2}\mid\nobreak A_{1}\wedge\nobreak X$,
and $y_{3}=A_{3}\mid A_{2}\wedge A_{1}\wedge X$ for some $A_{1}$,
$A_{2}$, $A_{3}$, and $X$, is dense in $[0,1]^{3}$.
\item [{C5.}] There exists a continuous function $F\colon[\mathsf{f},\mathsf{t}]^{2}\rightarrow[\mathsf{f},\mathsf{t}]$,
strictly increasing in both arguments on $(\mathsf{f},\mathsf{t}]^{2}$,
such that $A\wedge B\mid X=F\left(A\mid B\wedge X,\,B\mid X\right)$
for any $A$, $B$, and satisfiable $X$. Here we use $\mathsf{t}\triangleq A\mid X$
for any tautology $A$, and $\mathsf{f}\triangleq S\left(\mathsf{t}\right)$.
\end{lyxlist}
These requirements have not been without controversy. For example,
Shafer \cite{shafer-response} objects to C1, C3, and C5; Halpern
\cite{halpern-99a,halpern-99b} questions C4; and Colyvan \cite{colyvan}
objects to C2 on the basis that it presumes the law of the excluded
middle.

Our approach has no equivalent of C1, C3, C4, or C5:
\begin{itemize}
\item We are agnostic as to the set of allowed plausibility values. We do
not even require that plausibility values be totally ordered. 
\item We have no requirement on how the plausibility $\neg A$ decomposes.
\item We have no density requirement on plausibility values.
\item We have no requirement on how the plausibility of $A\wedge B$ decomposes,
much less any continuity or strictness requirements for such decompositions.
\end{itemize}
We do retain a variant of C2. Our goal is to extend the \emph{classical}
propositional logic; we make no attempt to address intuitionistic
logic.

Our requirements are all based on preserving in the extended logic
some existing property of CPL. Three of these are invariances\textemdash ways
in which $A$ or $X$ may be modified without altering $A\mid X$\textemdash and
the last is a requirement to preserve those distinctions in degree
of plausibility already present in CPL. We believe that such an approach
leaves far less room for objections to the requirements.

The results we obtain are similar to those of Cox's Theorem, with
these differences:
\begin{itemize}
\item We obtain an order isomorphism $P$ between plausibility values and
\emph{rational} probability values.
\item We obtain specific numerical values for $P\left(A\mid X\right)$,
and not just the laws for decomposing $P\left(A\mid X\right)$.
\item The conditioning information $X$ is necessarily a propositional formula.
(Some variants \cite{clayton-waddington,ptlos,vanhorn} of Cox's Theorem
allow $X$ to be an undefined ``state of information'' to which
we may add additional propositional information.)
\item Our results apply only to finite problem domains or finite approximations
of infinite domains. (In Section \ref{subsec:Infinite-Domains} we
discuss how to extend our results to infinite domains.)
\end{itemize}

\subsection{Clayton and Waddington: bridging the intution gap}

In a recent paper, Clayton and Waddington \cite{clayton-waddington}
seek to ``bridge the intuition gap'' in Cox's Theorem by proposing
alternative requirements they argue are more intuitively reasonable,
and then \emph{proving} C4 and the strictness of $F$ in C5 as theorems.
 We find their most interesting and important contributions to be
the following:
\begin{itemize}
\item Rather than just tweaking Cox's Theorem, they have instead created
an entirely new proof of its result. The meat of the proof of Cox's
Theorem lies in deriving functional equations that $F$ and $S$ must
satisfy, then solving those equations. But by the time they have ``bridged
the intuition gap'' in the requirements, they already have most of
the proof completed, without any need to solve functional equations.
\item Jaynes \cite{ptlos} argues that, if one's background information
is ``indifferent'' between two propositions, then they should be
assigned equal plausibility. They formalize this idea as an invariance
principle: if $A'\mid X'$ is obtained from $A\mid X$ by consistently
renaming all propositional symbols used via a one-to-one mapping,
then the two plausibilities are equal.
\item Their proof yields the classical \emph{definition} of probability\textemdash the
ratio of the number of positive cases to the number of all possible
cases\textemdash as a \emph{theorem} in certain cases they call ``$N$-urns.''
\end{itemize}
However, the list of assumptions they use is fairly long: 11 altogether.
We obtain similar results with fewer and simpler requirements, given
in Sections \ref{sec:Invariance-from-logical-equiv}\textendash \ref{sec:Ordering-Properties}.
Here is a comparison:
\begin{itemize}
\item Clayton and Waddington use a variant of C2 that replaces ``$A$ is
logically equivalent to $A'$'' with ``$A\wedge X$ is logically
equivalent to $A'\wedge X$''; this and their Assumption 2.1 are
comparable to our R\ref{req:equivalence} (replacement of premise
or query with a logically equivalent formula).
\item Our R\ref{req:ordering} (preservation of the implication ordering)
is a stronger (more general) version of their Assumption 2.7. We do
not actually need this more general form\textemdash Lemma \ref{lem:upsilon-is-strictly-increasing}
only uses a restricted form that corresponds to their Assumption 2.7\textemdash but
we feel that the rationale for the requirement is more clearly seen
in this more general form.
\item Our R\ref{req:definitions} (invariance under definition of new propositional
symbols) and R\ref{req:irrelevance} (invariance under addition of
irrelevant information) have no direct equivalent among their assumptions,
but are inspired by their discussion of the Principle of Indifference
and their Assumption 2.3 (translation invariance).
\item We use no equivalent of C1 nor their Assumptions 1.2, 1.3, 1.4, 3.2,
3.3, and 3.5.
\end{itemize}
From the above comparison one can see that it is the replacement of
Assumption 2.3 with R\ref{req:definitions} and R\ref{req:irrelevance}
that allows a drastic pruning of the assumptions used in obtaining
the main results of their paper.

Our approach was inspired by Clayton and Waddington's notion of an
``$N$-urn'' and the results they prove for $N$-urns. Our most
important innovation is Lemma \ref{lem:sample-space} showing how
to reduce \emph{every} allowable query-premise pair to a certain kind
of $N$-urn.

\subsection{Carnap: logical probability\label{sec:Carnap-logical-probability}}

Carnap \cite{carnap} undertakes an extensive investigation of ``logical
probability.'' He mentions two notions of probability: $\mathrm{probability}_{1}$
is epistemic probability\footnote{Carnap later favored a decision-theoretic view of $\mathrm{probability}_{1}$
\cite[Preface to the Second Edition]{carnap}.} (``the degree of confirmation of a hypothesis $h$ with respect
to an evidence statement $e$''), and $\mathrm{probability}_{2}$
is relative frequency. His focus is on $\mathrm{probability}_{1}$
and the problem of induction. Our terminology and his correspond roughly
as follows:
\begin{itemize}
\item Instead of a ``plausibility function'' Carnap discusses a ``confirmation
function'' $c$.
\item The rough equivalent of $A\mid X$ in Carnap's system is $c(A,X)$,
with a crucial difference described below.
\item We call $A$ and $X$ the query and premise, respectively; he calls
them the hypothesis and evidence.
\end{itemize}
We take $A$ and $X$ to be propositional formulas, whereas Carnap
allows them to be sentences in a variant of first-order predicate
logic in which the only allowed terms are variables and constant symbols.
The domain of discourse is taken to be a countably infinite set of
individuals, and for each individual there is a corresponding constant
symbol. He calls this most general form of the language $\mathfrak{L}_{\infty}$. 

Carnap also considers restricted languages $\mathfrak{L}_{N}$, $N\geq1$,
in which only the first $N$ constant symbols are allowed and the
domain of discourse is only the first $N$ individuals. Most of the
focus is on the finite languages $\mathfrak{L}_{N}$, with confirmation
functions for $\mathfrak{L}_{\infty}$ defined via a limiting process
on the sequence of languages $\mathfrak{L}_{N}$, $N\geq1$. This
is similar in spirit, though not in detail, to our approach to infinite
domains. 

A difference between Carnap's approach and ours is that we posit a
single language and single set of propositional symbols $\mathcal{S}$
to be used for all problem domains, whereas in Carnap's system each
problem domain has its own language $\mathfrak{L}_{\infty}$ with
its own set of predicate symbols and associated arities and interpretations.

Carnap's finite languages $\mathfrak{L}_{N}$ are equivalent to propositional
languages with a finite number of propositional symbols. Let an atomic
sentence be a formula of the form $p\left(s_{1},\ldots,s_{k}\right)$
for some $k$-ary predicate symbol $p$ and constant symbols $s_{1},\ldots,s_{k}$.
There are a finite number of distinct atomic sentences in $\mathfrak{L}_{N}$.
Define the propositional language $\mathfrak{L}'_{N}$ to have as
propositional symbols the atomic sentences of $\mathfrak{L}_{N}$.
We can transform any sentence $A\in\mathfrak{L}_{N}$ into an equivalent
\emph{propositional} formula $A'\in\mathfrak{L}'_{N}$:
\begin{enumerate}
\item Remove all quantifiers by repeatedly replacing any occurrence of a
subformula $\forall x\,\varphi(x)$ with the semantically equivalent
finite conjunction $\varphi\left(s_{1}\right)\wedge\cdots\wedge\varphi\left(s_{N}\right)$,
where $s_{1},\ldots,s_{N}$ are the constant symbols for the first
$N$ individuals.
\item If equality is allowed, replace any subformula $s_{i}=s_{i}$ with
a logically valid sentence, and any subformula $s_{i}=s_{j}$, $i\neq j$,
with an unsatisfiable sentence. In both cases choose replacement sentences
that contain no quantifiers nor use of equality. (Carnap specifies
that distinct constant symbols are assumed to reference distinct individuals.)
\end{enumerate}
We end up with propositional formulas in which the propositional symbols
have internal structure, but this internal structure is of no consequence
from the standpoint of deductive logic. To have logical effect, any
intended meaning for this internal structure must be given by additional
formulas in $\mathfrak{L}'_{N}$, axioms that are added to the set
of premises used. For example, suppose that we are deducing the logical
consequences of a set of premises $\Gamma$, that we have a two-place
predicate $\mathrm{lt}$, and that the intended meaning of $\mathrm{lt}(x,y)$
is that individual $x$ precedes individuals $y$ in some total ordering.
Then we must add to $\Gamma$ a set of axioms such as the following:
\begin{gather*}
\forall x\,\neg\mathrm{lt}(x,x)\\
\forall x\,\forall y\,\left(\mathrm{lt}(x,y)\vee\mathrm{lt}(y,x)\vee x=y\right)\\
\forall x\,\forall y\,\forall z\,\left(\mathrm{lt}(x,y)\wedge\mathrm{lt}(y,z)\rightarrow\mathrm{lt}(x,z)\right)
\end{gather*}
More precisely, we must add to $\Gamma$ the result of transforming
the above sentences of $\mathfrak{L}_{N}$ into equivalent \emph{propositional}
formulas of $\mathfrak{L}'_{N}$.

Carnap, however, is unwilling to axiomatize the intended interpretations
of the predicate symbols in this way. He writes \cite[p. 55]{carnap},
\begin{quote}
Since we intend to construct inductive logic as a theory of degree
of confirmation, based upon the meanings of the sentences involved\textemdash \emph{in
contradistinction to a mere calculus}\textemdash we shall construct
the language systems $\mathfrak{L}$ with an interpretation, hence
as a system of semantical rules, not as uninterpreted syntactical
systems.
\end{quote}
(Emphasis added.) This is an important difference between Carnap's
system and ours, one that we discuss further in Section \ref{sec:The-Plausibility-Function}
and Section \ref{subsec:Uniform-Versus-Non-uniform}. Thus the equivalent
of Carnap's $c(h,e)$ in our scheme is not $h\mid e$, but instead
$h\mid e\wedge X$, where $X$ is a conjunction of
\begin{itemize}
\item propositional axioms expressing the logical structure of the problem
domain, and
\item propositional formulas expressing any other background information. 
\end{itemize}
Unlike Cox \cite{cox}, Carnap makes no attempt to derive the laws
of probability from more fundamental considerations; instead, his
``conventions on adequacy'' \cite[p. 285]{carnap} require that
a valid confirmation function should conform to the laws of probability.
He ensures this by proceeding as follows:
\begin{enumerate}
\item A \emph{state description} for $\mathfrak{L}_{N}$ amounts to a truth
assignment on the set of atomic sentences of $\mathfrak{L}_{N}$.
\item A \emph{regular measure function} for $\mathfrak{L}_{N}$ amounts
to a strictly positive probability mass function over the set of state
descriptions for $\mathfrak{L}_{N}$.
\item A \emph{regular confirmation function} $c$ for $\mathfrak{L}_{N}$
is then a confirmation function defined as 
\[
c(h,e)=m(h\wedge e)/m(e)
\]
for some regular measure function $m$ on $\mathfrak{L}_{N}$.
\item These notions are extended to $\mathfrak{L}_{\infty}$ by imposing
a consistency condition on a sequence of regular measure functions
$m_{N}$ on $\mathfrak{L}_{N}$, $N\geq1$, considering the associated
regular confirmation functions $c_{N}$, and defining $c_{\infty}(h,e)=\lim_{N\rightarrow\infty}c_{N}(h,e)$.
\end{enumerate}
We see therefore that, although it is the conditional probabilities
$c(h,e)$ that most interest Carnap, unconditional probabilities are
for him more fundamental. In contrast, we take conditional plausibilities
as the fundamental concept and, rather than imposing the laws of probability,
seek to derive them.

Carnap's goal is that the confirmation function appropriate to a problem
domain should be uniquely determined by the semantics of the language
$\mathfrak{L}_{\infty}$, specifically, the intended interpretation
of the predicate symbols and constant symbols. In this he fails. Limiting
his attention to systems that contain monadic predicates (``properties'')
only, he proposes a specific confirmation function $c^{*}$, but says
\cite[p. 563]{carnap},
\begin{quote}
Now the chief arguments in favor of the function $c^{*}$\ldots{}
will consist in showing that this function is free of the inadequacies
in the other methods. It may then still be inadequate in other respects.
It will not be claimed that $c^{*}$ is a perfectly adequate explicatum
for $\mbox{probability}_{1}$, let alone that it is the only adequate
one\ldots{}
\end{quote}
He later \cite[Preface to Second Edition]{carnap}\cite{carnap-continuum}
proposes instead an entire family of confirmation functions $c_{\lambda}$
parameterized by a positive number $\lambda$.

In contrast, we show in Theorem \ref{thm:classical-prob-def} that
there is (up to isomorphism) a \emph{single}, \emph{unique} plausibility
function satisfying our criteria. Ironically, it corresponds to the
one confirmation function explicitly rejected by Carnap: a uniform
distribution over the set of truth assignments satisfying the premise
/ evidence. We discuss this further in Section \ref{subsec:Uniform-Versus-Non-uniform}.

\section{Logical Preliminaries\label{sec:Logical-Preliminaries}}

In this section we review some concepts from CPL, introduce some additional
logical concepts of our own, and discuss the nature of the plausibility
function as extending the logical consequence relation $\models$.

\subsection{Review of classical propositional logic}

A \emph{proposition} is a statement or assertion that must be true
or false; it is \emph{atomic} if it cannot be decomposed into simpler
assertions.

A \emph{propositional symbol} is one of a countably infinite set of
symbols $\mathcal{S}$ that are used to represent atomic propositions.
We abbreviate this to just ``symbol'' when the meaning is clear.

If $S\subseteq\mathcal{S}$ then a \emph{propositional formula on
$S$} is one of the following:
\begin{itemize}
\item a propositional symbol from $S$;
\item a formula $\neg A$, meaning ``not $A$'', for some propositional
formula $A$ on $S$; or
\item a formula $A\wedge B$, meaning ``$A$ and $B$,'' where $A$ and
$B$ are propositional formulas on $S$.
\end{itemize}
The other common logical operators\textemdash $A\vee B$ (or), $A\rightarrow B$
(implies), and $A\leftrightarrow B$ (if and only if)\textemdash are
defined in terms of $\neg$ and $\wedge$ in the usual way. We abbreviate
``propositional formula'' as just ``formula'' when the meaning
is clear.

We write $\Phi\left(S\right)$ for the set of all propositional formulas
on $S$, and $\Phi^{+}(S)$ for the satisfiable formulas.

We write $\sigma\left\llbracket A_{1},\ldots,A_{n}\right\rrbracket $
for the set of all propositional symbols occuring in any of the propositional
formulas $A_{1},\ldots,A_{n}$.

A \emph{truth assignment on $S$} is a function $\rho\colon S\rightarrow\{0,1\}$,
with 0 and 1 standing for falsity and truth, respectively. We recursively
extend it to all formulas on $S$ in the obvious way: 
\begin{eqnarray*}
\rho\left\llbracket A\right\rrbracket  & = & \rho(A)\mbox{ if }A\in S\\
\rho\left\llbracket \neg A\right\rrbracket  & = & 1-\rho\left\llbracket A\right\rrbracket \\
\rho\left\llbracket A\wedge B\right\rrbracket  & = & \rho\left\llbracket A\right\rrbracket \cdot\rho\left\llbracket B\right\rrbracket 
\end{eqnarray*}
where `$\cdot$' is just integer multiplication.

A truth assignment $\rho$ on $S$ \emph{satisfies} a formula $A$
on $S$ if $\rho\left\llbracket A\right\rrbracket =1$. A formula
$A$ is \emph{satisfiable} if there is $some$ truth assignment $\rho$
on $\sigma\left\llbracket A\right\rrbracket $ that satisfies $A$.
A formula $A$ is \emph{logically valid}, written $\models A$,\emph{
}if $every$ truth assignment on $\sigma\left\llbracket A\right\rrbracket $
satisfies $A$.

We say that $B$ is a \emph{logical consequence} of $A$, or \emph{$A$
logically implies $B$}, written $A\models B$, if every truth assignment
on $\sigma\left\llbracket A,B\right\rrbracket $ that satisfies $A$
also satisfies $B$. This captures the notion of a logically valid
argument: conclusion $B$ follows as a logical consequence of premises
$A_{1},\ldots,A_{n}$ if $A_{1}\wedge\cdots\wedge A_{n}\models B$.
Note that $A\models B$ if and only if $\models A\rightarrow B$.

The \emph{restriction} of a truth assignment $\rho$ on $S$ to some
$S'\subseteq S$ is the truth assignment $\rho'$ on $S'$ such that
$\rho'(s)=\rho(s)$ for every $s\in S'$.

Note that $\rho\left\llbracket A\right\rrbracket $ depends only on
the truth values assigned to those symbols that actually appear in
$A$; if $A$ is a formula on $S'\subseteq S$, $\rho$ is a truth
assignment on $S$, and $\rho'$ is the restriction of $\rho$ to
$S'$, then $\rho'\left\llbracket A\right\rrbracket =\rho\left\llbracket A\right\rrbracket $.
Because of this, we have some leeway in choosing the set of propositional
symbols to use in the definitions of ``logically valid,'' ``satisfiable,''
and ``logical consequence.'' If $\sigma\left\llbracket A\right\rrbracket \subseteq S_{A}$
and $\sigma\left\llbracket A,B\right\rrbracket \subseteq S_{AB}$,
then
\begin{itemize}
\item $A$ is logically valid iff every truth assignment on $S_{A}$ satisfies
$A$.
\item $A$ is satisfiable iff some truth assignment on $S_{A}$ satisfies
$A$.
\item $A\models B$ iff every truth assignment on $S_{AB}$ that satisfies
$A$ also satisfies $B$.
\end{itemize}
\global\long\def\pleqv#1{\equiv_{{\scriptscriptstyle #1}}\!}
Some notation.
\begin{itemize}
\item Two formulas are \emph{logically equivalent}, written $A\equiv B$,
if $\models A\leftrightarrow B$.
\item $A$ and $B$ are \emph{logically equivalent assuming $X$}, written
$A\pleqv XB$, if $X\models A\leftrightarrow B$.
\item We will use finite quantification as an abbreviation where convenient,
for example writing $\bigwedge_{i=1}^{n}A_{i}$ for $A_{1}\wedge\cdots\wedge A_{n}$.
\item If we write $A_{1}\wedge\cdots\wedge A_{n}$ and $n=0$, we understand
this to mean some logically valid formula such as $s\vee\neg s$.
\item If we write $A_{1}\vee\cdots\vee A_{n}$ and $n=0$, we understand
this to mean some unsatisfiable formula such as $s\wedge\neg s$.
\end{itemize}

\subsection{Finite sample spaces}

The development of probability theory usually begins with the idea
of a sample space, which has been downplayed so far\textemdash the
focus has been on propositions. We find in Theorem \ref{thm:classical-prob-def}
that $A\mid X$ can be characterized in terms of an induced sample
space: the set of truth assignments that satisfy $X$.\footnote{This is similar to Carnap's system, in which the set of state descriptions
serve as a sample space.} In particular, we find that $A\mid X$ is a function of the proportion
of points from this induced sample space that satisfy $A$. This motivates
the following:
\begin{defn}
$\#_{S}(X)$ is the number of truth assignments on $S$ satisfying
$X$, for any formula $X$ and finite $S$ such that $\sigma\left\llbracket X\right\rrbracket \subseteq S\subseteq\mathcal{S}$.
\end{defn}
The size of the induced sample space is $\#_{S}(X)$, and the proportion
of points from the induced sample space that satisfy $A$ is $\#_{S}\left(A\wedge X\right)/\#_{S}(X)$,
for $S\supseteq\sigma$$\left\llbracket A,X\right\rrbracket $.

Typically one thinks of a finite sample space as an arbitrary set
of $n>0$ distinct values $\Omega=\left\{ \omega_{1},\ldots,\omega_{n}\right\} $
representing different possible states of some system under consideration.
We can relate this to our notion of an induced sample space by choosing
a set of $n$ propositional symbols $S=\left\{ s_{1},\ldots,s_{n}\right\} $,
with the intended interpretation of $s_{i}$ being that the state
of the system is $\omega_{i}$. If our premise $X$ is a formula expressing
that exactly one of the $s_{i}$ is true, then there is a one-to-one
correspondence between our original sample space $\Omega$ and the
induced sample space of truth assignments, with $\omega_{i}$ corresponding
to the single truth assignment $\rho$ on $S$ satisfying $s_{i}\wedge X$.
This motivates the following:
\begin{defn}
\label{def:n-urn}Given any sequence of $n>0$ propositional symbols
$s_{1},\ldots,s_{n}$, 
\[
\left\langle s_{1},\ldots,s_{n}\right\rangle =\left(s_{1}\vee\cdots\vee s_{n}\right)\wedge\bigwedge_{1\leq i<j\leq n}\neg\left(s_{i}\wedge s_{j}\right);
\]
that is, $\left\langle s_{1},\ldots,s_{n}\right\rangle $ means that
exactly one of the $s_{i}$ is true. 
\end{defn}

\subsection{The implication ordering\label{sec:implication-ordering}}

At first blush it would seem that CPL tells us very little about the
relative plausibilities of different propositions, beyond determining
which are certainly true and which are certainly false given a premise
$X$. The reality is quite the opposite: CPL comes equipped with a
rich inherent plausibility ordering that we call the \emph{implication
ordering}.

\global\long\def\plle#1{\preceq_{{\scriptscriptstyle #1}}\!}
\global\long\def\pllt#1{\prec_{{\scriptscriptstyle #1}}\!}

\begin{defn}
Let $A,B,X\in\Phi\left(\mathcal{S}\right)$ with $X$ satisfiable.
We define 
\begin{eqnarray*}
\left(A\plle XB\right) & \Leftrightarrow & \left(X\models A\rightarrow B\right)\\
\left(A\pllt XB\right) & \Leftrightarrow & \left(A\plle XB\right)\mbox{ and not }\left(B\plle XA\right).
\end{eqnarray*}
\end{defn}
The following properties are easily verified:
\begin{itemize}
\item The relation $\plle X$ is a preorder: it is reflexive and transitive,
but not anti-symmetric.
\item $A\pllt XB$ is the same as ($A\plle XB$ and not $A\pleqv XB$).
\item $A\pleqv XB$ if and only $A\plle XB$ and $B\plle XA$.
\item If $A\pleqv XA'$ and $B\pleqv XB'$ and $A\plle XB$ then $A'\plle XB'$.
\end{itemize}
Hence $\plle X$ defines a partial order on the \emph{equivalence
classes} of propositional formulas under the relation $\pleqv X$.
This partial order is essentially just the subset ordering on truth
assignments:
\begin{property*}
Let $A_{1},A_{2},X\in\Phi\left(\mathcal{S}\right)$ with $X$ satisfiable.
Then $A_{1}\plle XA_{2}$ if and only if $\alpha\left\llbracket A_{1}\right\rrbracket \subseteq\alpha\left\llbracket A_{2}\right\rrbracket $,
where $\alpha\left\llbracket A\right\rrbracket $ is the set of truth
assignments on $\mathcal{S}$ that satisfy both $A$ and $X$.
\end{property*}
Assuming $X$, we therefore conclude the following:
\begin{itemize}
\item If $A\plle XB$ then $B$ is at least as plausible as $A$, since
$B$ is true for any possible world (truth assignment satisfying $X$)
for which $A$ is true.
\item If $A\equiv_{X}B$ then $A\plle XB$ and $B\plle XA$, hence $A$
and $B$ are equally plausible.
\item If $A\pllt XB$ then $B$ is strictly \emph{more} plausible than $A$,
since there are possible worlds for which $B$ is true and $A$ is
not, but not vice versa.
\end{itemize}
Consider an example that uses three propositional symbols $s_{1},s_{2},s_{3}$
with $X$ defined to be the formula stating that exactly one of these
three is true: $X=\left\langle s_{1},s_{2},s_{3}\right\rangle $.
Let $F$ be any unsatisfiable formula and $T$ be any logically valid
formula. Then 
\[
F\pllt Xs_{1}\pllt X\left(s_{1}\vee s_{2}\right)\pllt X\left(s_{1}\vee s_{2}\vee s_{3}\right)\pleqv XT.
\]
Note that adding additional information to the premise yields additional
formulas $A\rightarrow B$ as logical consequences, and hence may
collapse previously distinct plausibilities. Continuing the example,
if we add additional information to $X$ to obtain $Y=X\wedge\neg s_{2}$,
then 
\[
F\pllt Ys_{1}\pleqv Y\left(s_{1}\vee s_{2}\right)\pllt Y\left(s_{1}\vee s_{2}\vee s_{3}\right)\pleqv YT.
\]

\subsection{The plausibility function\label{sec:The-Plausibility-Function}}

We extend CPL to Jaynes's ``logic of plausible reasoning'' by introducing
a \emph{plausibility function} $\left(\cdot\mid\cdot\right)$ whose
domain is $\Phi\left(\mathcal{S}\right)\times\Phi^{+}\left(\mathcal{S}\right)$.
Think of $\left(\cdot\mid\cdot\right)$ as extending the logical consequence
relation: whereas $X\models A$ means that $A$ is known true (given
$X$), and $X\models\neg A$ means that $A$ is known false, it may
be that neither of these relations hold; $\left(\cdot\mid\cdot\right)$
fills in the gaps, so to speak, by assigning intermediate plausibilities
in such a case.

The logical consequence relation, as we have defined it, takes only
a single premise $X$ on the left-hand side, rather than a set of
premises $\mathcal{X}$. The compactness theorem for CPL says that
if $A$ is a logical consequence of a set of premises $\mathcal{X}$,
then it is a logical consequence of a finite subset of $\mathcal{X}$
\cite[p. 16]{mathematical-logic}; but any finite set of premises
$X_{1},\ldots,X_{n}$ can be combined into a single premise $X=X_{1}\wedge\cdots\wedge X_{n}$.
Likewise, the plausibility function $\left(\cdot\mid\cdot\right)$
takes only a single premise as its second argument.

We write $\mathbb{P}$ for the range of the plausibility function,
but leave it otherwise unspecified:
\begin{defn}
$\mathbb{P}$ is the set of achievable and meaningful plausibility
values; that is, 
\[
\mathbb{P}=\left\{ \left(A\mid X\right)\colon A,X\in\Phi\left(\mathcal{S}\right)\mbox{ and }X\mbox{ is satisfiable}\right\} .
\]
\end{defn}
There has been much unnecessary controversy over Cox's Theorem due
to differing implicit assumptions as to the nature of its plausibility
function. Halpern \cite{halpern-99a,halpern-99b} claims to demonstrate
a counterexample to Cox's Theorem by examining a finite problem domain,
but his argument presumes that there is a \emph{different} plausibility
function for every problem domain. Others \cite{cox,paris} seem to
presume a single plausibility function, but with domain-specific information
serving as an implicit extra argument\footnote{Strictly speaking, this also amounts to a different plausibility function
for every problem domain, but the practical difference is that certain
structural properties of the plausibility function, such as its range
and the choice of the functions $F$ and $S$, remain the same across
problem domains.}. A third interpretation \cite{clayton-waddington,ptlos,vanhorn}
presumes a single plausibility function with \emph{all} relevant information
about the problem domain encapsulated in the second argument, the
``state of information.''

We follow this third interpretation, with the premise\textemdash a
propositional formula\textemdash serving as the state of information:
\begin{itemize}
\item In CPL there is only a \emph{single} logical consequence relation
$\models$, defined on $\Phi\left(\mathcal{S}\right)$, rather than
entirely different logical consequence relations for each problem
domain. In our extended logic there is likewise only a \emph{single}
plausibility function, defined on $\Phi\left(\mathcal{S}\right)\times\Phi^{+}\left(\mathcal{S}\right)$,
and a single set of plausibility values $\mathbb{P}$ that are used
for \emph{all} problem domains.
\item In CPL any information about the problem domain that we wish to use
for deduction must be included in the premise(s) to the logical consequence
relation. Likewise in our extended logic, all relevant background
information about the problem domain must be included in the premise
to the plausibility function.
\end{itemize}
Think of the plausibility function as something one could implement
as a pure function in some programming language, taking as input two
strings matching the grammar for propositional formulas (or their
corresponding parse trees), and having access to \emph{no other} source
of information about the problem domain.

Note that we use the same set of propositional symbols $\mathcal{S}$
for all problem domains, rather than having a different set of propositional
symbols for each problem domain. The latter option would make the
set of allowed propositional symbols an implicit extra argument to
the plausibility function. The set of symbols $\mathcal{S}$ is countably
infinite to allow modeling arbitrarily complex problem domains.

As an example of incorporating background information into the premise,
suppose that we wish to discuss the outcome of rolling a six-sided
die, and our background knowledge is simply the list of distinct possible
outcomes. Let symbols $s_{i}$, $1\leq i\leq6$, have the intended
interpretation that the outcome is $i$. The formula $\left\langle s_{1},\ldots,s_{6}\right\rangle $
expresses our background knowledge, and so 
\[
s_{2}\mid\left\langle s_{1},\ldots,s_{6}\right\rangle 
\]
is the plausibility of rolling a 2, and 
\[
s_{1}\vee s_{2}\mid\left(s_{1}\vee s_{3}\vee s_{5}\right)\wedge\left\langle s_{1},\ldots,s_{6}\right\rangle 
\]
is the plausibility of rolling a 1 or 2 given that the outcome is
odd.

This stands in stark contrast to Carnap, who as previously mentioned
rejects such an axiomatic approach. The second argument to his confirmation
function is the \emph{evidence}, which is ``an observational report''
that ``refer{[}s{]} to facts'' \cite[pp. 19--20]{carnap}; background
knowledge about the meaning of the symbols and logical structure of
the domain is excluded. Given a situation like our die roll, in which
there is a ``family of related properties,'' exactly one of which
holds true for each individual, Carnap goes so far as to require modifying
the definition of a fundamental concept in his system, the state-description,
rather than simply including this information in the evidence \cite[p. 77]{carnap}.

\section{Invariance from Logical Equivalence\label{sec:Invariance-from-logical-equiv}}

In this and the following sections we introduce our Requirements on
the plausibility function and prove their consequences. These are
all based on preserving existing properties of CPL. We shall consider
properties of the logical consequence relation $\models$, as well
as the implication ordering $\plle X$ for a given premise $X$. 

The first property we consider is invariance under replacement of
premise or query by a logically equivalent formula.

The logical consequence relation $\models$ is invariant to replacement
of premise by a logically equivalent formula: if $X\equiv Y$ then
for all formulas $A$ we have $X\models A$ if and only if $Y\models A$.
We require that the plausibility function exhibit this same invariance.
This may be further justified by noting that the implication orderings
$\plle X$ and $\plle{_{Y}}$ are identical when $X\equiv Y$. 

The relation $\models$ is also invariant to replacement of conclusion
by a logically equivalent formula. In fact, the replacement formula
need only be logically equivalent \emph{assuming} the premise: if
$A\pleqv XB$, then $X\models A$ if and only if $X\models B$. We
require that the plausibility function exhibit this invariance also,
for the query. This may be further justified by our argument in Section
\ref{sec:implication-ordering} that we should consider $A$ and $B$
equally plausible, assuming $X$, whenever $A\pleqv XB$.

We combine these into a single requirement:
\begin{requirement}
\label{req:equivalence}If $X\equiv Y$ and $A\equiv_{X}B$ then $A\mid X=B\mid Y$.
\end{requirement}

\section{Invariance under Definition of New Symbols \label{sec:Invariance-under-definition}}

It is common in mathematical proofs to define new symbols as abbreviations
for complex expressions or formulas. The same may be done in propositional
logic: we may introduce a new propositional symbol $s$ (that appears
in neither the premises nor conclusion) and use it as an abbreviation
for some complex propositional formula $E$, by adding the definition
$s\leftrightarrow E$ to our premises. This does not invalidate any
logical consequence we already had, nor any create any new logical
consequence that does not mention $s$.

Specifically, let $s$ be a symbol not occurring in $X$, $E$, or
$A$, and define $Y=\left(s\leftrightarrow E\right)\wedge X$ . Then
$X\models A$ if and only if $Y\models A$, and consequently, $\plle X$
and $\plle Y$ are identical on $\Phi\left(\mathcal{S}\setminus\left\{ s\right\} \right)$.
We require that the plausibility function exhibit the same invariance:
\begin{requirement}
\label{req:definitions}Let $s\in\mathcal{S}$ but $s\notin\sigma\left\llbracket A,X,E\right\rrbracket $.
Then $A\mid X=A\mid\left(s\leftrightarrow E\right)\wedge X$.
\end{requirement}
One cannot evade the force of this Requirement by supposing a problem
domain with a limited set of symbols. Recall that there is only one
plausibility function, used for all problem domains, and that $\mathcal{S}$
is countably infinite. Furthermore, even if the plausibility function
were to take as a third argument a finite set of symbols from which
the query and premise are constructed, the notion of extending a domain
by defining an additional variable as a function of existing variables
would still make sense. Forbidding such extension would be an artificial
and unreasonable restriction, as one can already do this in CPL.

\subsection{Invariance under renaming}

To build some intuition for R\ref{req:definitions} we now explore
some of its more straightforward consequences, in conjunction with
R\ref{req:equivalence} (logical equivalence).

Let us write $B\left[s/C\right]$ for the result of replacing every
occurrence of symbol $s$ in formula $B$ with the formula $C$. If
$s$ and $t$ are distinct symbols, with $t$ not occurring in formulas
$A$ or $X$, then using R\ref{req:definitions} to introduce a definition
and later remove a different one gives us
\begin{align*}
A\mid X & =A\mid\left(t\leftrightarrow s\right)\wedge X\\
 & =A[s/t]\mid\left(t\leftrightarrow s\right)\wedge X[s/t]\\
 & =A[s/t]\mid\left(s\leftrightarrow t\right)\wedge X[s/t]\\
 & =A[s/t]\mid X[s/t].
\end{align*}
That is, we can rename any single symbol, replacing it throughout
$A$ and $X$ with a new symbol, and this leaves the plausibility
unchanged.

Repeating the process, the plausibility is invariant if we rename
any set of symbols $S=\left\{ s_{1},\ldots,s_{n}\right\} $ to new
symbols $T=\left\{ t_{1},\ldots,t_{n}\right\} $ not occurring in
$A$ or $X$. We can also permute the symbol names, by renaming from
$s_{1},\ldots,s_{n}$ to $t_{1},\ldots,t_{n}$ and then to a permuation
$s'_{1},\ldots,s'_{n}$ of $s_{1},\ldots,s_{n}$. That is, if we write
$B\left[s_{1}/C_{1},\ldots,s_{n}/C_{n}\right]$ for the formula obtained
by \emph{simultaneously} replacing each symbol $s_{i}$ with the formula
$C_{i}$, we have 
\[
A\mid X=A\left[s_{1}/s'_{1},\ldots,s_{n}/s'_{n}\right]\mid X\left[s_{1}/s'_{1},\ldots,s_{n}/s'_{n}\right].
\]
This \emph{result} is the same as Clayton \& Waddington's \emph{Assumption}
2.3 (translation invariance) \cite{clayton-waddington}, which they
motivate via Jaynes's ``indifference'' criterion \cite[p. 19]{ptlos}:
\begin{quote}
The robot always represents equivalent states of knowledge by equivalent
plausibility assignments. That is, if in two problems the robot's
state of knowledge is the same (except perhaps for the labeling of
the propositions), then it must assign the same plausibilities in
both.
\end{quote}
For example, if $a,b,c,d$ are distinct symbols, then the following
equalities hold: 
\begin{align*}
a\mid a\vee b & =c\mid c\vee d\\
a\mid a\rightarrow b & =b\mid b\rightarrow a
\end{align*}
Consider specifically the case where $X$ treats symbols $s$ and
$t$ symmetrically: that is, $X$ is logically equivalent to $X'=X[s/t,t/s]$.
One example would be
\begin{align*}
X & =\left(s\vee t\right)\wedge\neg\left(s\wedge t\right)\\
X' & =\left(t\vee s\right)\wedge\neg\left(t\wedge s\right).
\end{align*}
In this case we find that $s$ and $t$ must be equally plausible:
\[
s\mid X=t\mid X'=t\mid X.
\]
This result is similar in spirit to the principle of insufficient
reason: our premise $X$ provides no information that differs between
$s$ and $t$, so intuition suggests these propositions should be
equally plausible. The result is more general, however, in that $s$
and $t$ need not be mutually exclusive nor exhaustive.

Another transformation we can consider is that of replacing all occurrences
of symbol $s$ with $\neg s$ in both premise and query. As before,
let $s$ and $t$ be distinct symbols, with $t$ not occurring in
formulas $A$ nor $X$. Again we use R\ref{req:definitions} to introduce
a definition and later remove a different one; we also add a final
step that invokes the above-demonstrated invariance under renaming.
This yields the following:
\begin{align*}
A\mid X & =A\mid\left(t\leftrightarrow\neg s\right)\wedge X\\
 & =A[s/\neg\neg s]\mid\left(t\leftrightarrow\neg s\right)\wedge X[s/\neg\neg s]\\
 & =A[s/\neg t]\mid\left(t\leftrightarrow\neg s\right)\wedge X[s/\neg t]\\
 & =A[s/\neg t]\mid\left(s\leftrightarrow\neg t\right)\wedge X[s/\neg t]\\
 & =A[s/\neg t]\mid X[s/\neg t]\\
 & =A[s/\neg t][t/s]\mid X[s/\neg t][t/s]\\
 & =A[s/\neg s]\mid X[s/\neg s].
\end{align*}
That is, the plausibility is invariant to a transformation in which
we uniformly replace any single symbol with its negation throughout
both $A$ and $X$. In particular, if $X$ is logically equivalent
to $X[s/\neg s]$, then
\[
s\mid X=\neg s\mid X.
\]
This result may be viewed as an instance of the principle of insufficient
reason applied to the case of two indistinguishable possibilities.

Take note of the common pattern in the above two derivations:
\begin{enumerate}
\item Use R\ref{req:definitions} to introduce a definition of some symbol
$t$ in terms of symbol $s$ appearing in the premise or query.
\item Use logical equivalence and the definition of $t$ to rewrite premise
and query in a way that removes all occurrences of $s$ except its
occurrence in the right-hand side of the definition of $t$.
\item Use logical equivalence to rewrite the definition of $t$ in terms
of $s$ as a definition of $s$ in terms of $t$.
\item Use R\ref{req:definitions} to drop the definition of $s$, as this
symbol is now used nowhere else in the premise or query.
\end{enumerate}
Lemma \ref{lem:sample-space} in Section \ref{subsec:Reduction-to-canonical}
extends this pattern to \emph{sets} of symbols, simultaneously introducing
multiple definitions in step 1, and this yields a stronger form of
transformation invariance that subsumes the results derived here. 

\subsection{Invariance under change of variables}

Renaming symbols and swapping $s$ for $\neg s$ throughout both premise
and query are special cases of more general \emph{change of variables}
transformations. As an example of this, suppose that we are considering
a problem domain in which there is some quantity $x$ that can take
on any of $n$ discrete, ordered values $v_{1}<v_{2}<\cdots<v_{n}$.
There are two different vocabularies we might use for this domain:
\begin{enumerate}
\item Use symbols $s_{1},\ldots,s_{n}$ with the intended meaning of $s_{i}$
being ``$x=v_{i}$,'' and express ``$x\leq v_{i}$'' as $s_{1}\vee\cdots\vee s_{i}$.
\item Use symbols $t_{1},\ldots,t_{n}$ with the intended meaning of $t_{i}$
being ``$x\leq v_{i}$,'' and express ``$x=v_{i}$'' as $t_{i}\wedge\neg t_{i-1}$
when $i>1$, or just $t_{i}$ when $i=1$.
\end{enumerate}
The two vocabularies can express exactly the same propositions, so
there is no fundamental reason to choose one over the other, and it
seems that the plausibility $A\mid X$ should not depend on which
vocabulary we use. Going from one vocabulary to the other is just
a change of variables: we can express each of the $s_{i}$ in terms
of $t_{1},\ldots,t_{n}$, or we can express each of the $t_{i}$ in
terms of $s_{1},\ldots,s_{n}$.

This isn't quite enough, though. Defining
\begin{align*}
\tau_{\mathrm{st}}(A) & =A\left[s_{1}\,/\,t_{1},\,s_{2}\,/\,t_{2}\wedge\neg t_{1},\,\ldots,\,s_{n}\,/\,t_{n}\wedge\neg t_{n-1}\right]\\
\tau_{\mathrm{ts}}(B) & =B\left[t_{1}\,/\,s_{1},\,t_{2}\,/\,s_{1}\vee s_{2},\,\ldots,\,t_{n}\,/\,s_{1}\vee\cdots\vee s_{n}\right]
\end{align*}
we want $\tau_{\mathrm{st}}$ and $\tau_{\mathrm{ts}}$ to be inverses
of each other (up to logical equivalence). We find that $\tau_{\mathrm{ts}}\left(\tau_{\mathrm{st}}\left(A\right)\right)$
is logically equivalent to $A$, but 
\[
\tau_{\mathrm{st}}\left(\tau_{\mathrm{ts}}\left(B\right)\right)\equiv B\left[t_{1}\,/\,t_{1},\,t_{2}\,/\,t_{1}\vee t_{2},\ldots,\,t_{n}\,/\,t_{1}\vee\cdots\vee t_{n}\right]
\]
which is not, in general, logically equivalent to $B$. We need to
assume that $t_{i}\rightarrow t_{i+1}$ for $1\leq i<n$ to get the
desired equivalence. Such an assumption concords with the intended
meaning of $t_{i}$, and must be implied by the premise when vocabulary
2 is used. (Likewise, the premise must imply $\left\langle s_{1},\ldots,s_{n}\right\rangle $
when vocabulary 1 is used.) So this notion of change of variables
is more subtle than it appears at first glance; how do we define a
general rule that accounts for issues like this?

The solution is to define a change of variables in terms of a bijection
between
\begin{itemize}
\item the set of truth assignments satisfying the premise when vocabulary
1 is used, and
\item the set of truth assignments satisfying the premise when vocabulary
2 is used.
\end{itemize}
This motivates the following:
\begin{defn}
$f$ is a \emph{change-of-variables transformation} between the pairs
$(A,X)$ and $(A',X')$ if it is a bijection between
\begin{itemize}
\item the set of truth assignments on some $S\supseteq\sigma\left\llbracket A,X\right\rrbracket $
satisfying $X$, and
\item the set of truth assignments on some $S'\supseteq\sigma\left\llbracket A',X'\right\rrbracket $
satisfying $X'$,
\end{itemize}
with the additional property that any truth assignment $\rho$ on
$S$ satisfies $A\wedge X$ if and only if $f(\rho)$ satisfies $A'\wedge X'$.
\end{defn}
Note that the logical consequence relation trivially satisfies invariance
under change of variables:
\begin{quote}
$X\models A\Leftrightarrow X'\models A'$ if there exists a change-of-variables
transformation $f$ between $\left(A,X\right)$ and $\left(A',X'\right)$.
\end{quote}
For the plausibility function, invariance under change of variables
means the following:
\begin{quote}
$A\mid X=A'\mid X'$ if there exists a change-of-variables transformation
$f$ between $\left(A,X\right)$ and $\left(A',X'\right)$. 
\end{quote}
Invariance under definition of new symbols is a special case of invariance
under change of variables: we have $A'=A$, $X'=\left(s\leftrightarrow E\right)\wedge X$,
$S=\sigma\left\llbracket A,X,E\right\rrbracket $, $S'=S\cup\left\{ s\right\} $,
and $f(\rho)=\rho'$ where
\begin{align*}
\rho'(s) & =\rho\left\llbracket E\right\rrbracket \\
\rho'(t) & =\rho(t)\quad\mbox{if }t\in S.
\end{align*}
The inverse of $f$ maps $\rho'$ to the restriction of $\rho'$ to
$S$.

We show in Corollary \ref{cor:transformation-invariance} that R\ref{req:equivalence}
and R\ref{req:definitions} together imply invariance under change
of variables. So, given R\ref{req:equivalence}, invariance under
change of variables and invariance under definition of new symbols
are equivalent. We chose the latter as our requirement because it
is easier to explain and justify.

\subsection{Reduction to canonical form\label{subsec:Reduction-to-canonical}}

We take the first step towards our main result by showing that we
can reduce every query-premise pair to a canonical form in which the
premise merely states that we have a sample space of $n$ distinct
possibilities, and the query merely states that one of the first $m\leq n$
possibilities is true. In the following, keep in mind our convention
that $A_{1}\vee\cdots\vee A_{m}$ stands for some unsatisfiable formula
when $m=0$.
\begin{lem}
\label{lem:sample-space}Let $S\subseteq\mathcal{S}$ be finite, $A\in\Phi\left(S\right)$,
and $X\in\Phi^{+}(S)$. Then R\ref{req:equivalence} and R\ref{req:definitions}
together imply that 
\[
A\mid X=\left(t_{1}\vee\cdots\vee t_{m}\mid\left\langle t_{1},\ldots,t_{n}\right\rangle \right)
\]
where $n=\#_{S}(X)>0$, $m=\#_{S}\left(A\wedge X\right)\leq n$, and
$T=\left\{ t_{1},\ldots,t_{n}\right\} $ is any set of $n$ propositional
symbols disjoint from $S$.
\end{lem}
\begin{proof}
Let $\rho_{1},\ldots,\rho_{n}$ be the truth assignments on $S$ that
satisfy $X$, ordered so that the first $m$ also satisfy $A$. Enumerate
the elements of $S$ as $s_{1},\ldots,s_{p}$. The proof proceeds
in four steps.

\emph{Step 1. }For each $1\leq i\leq n$ and $1\leq j\leq p$ define
\begin{eqnarray*}
Z_{i} & = & L_{i,1}\wedge\cdots\wedge L_{i,p}\\
L_{i,j} & = & \begin{cases}
s_{j} & \mbox{if }\rho_{i}\left(s_{j}\right)=1\\
\neg s_{j} & \mbox{if }\rho_{i}\left(s_{j}\right)=0.
\end{cases}
\end{eqnarray*}
Note that $\rho_{i}$ is the one and only truth assignment on $S$
that satisfies $Z_{i}$.

Define the formulas 
\begin{eqnarray*}
D_{\mathrm{t},i} & = & t_{i}\leftrightarrow Z_{i}\\
D_{\mathrm{t}} & = & D_{\mathrm{t},1}\wedge\cdots\wedge D_{\mathrm{t},n}.
\end{eqnarray*}
Then by R\ref{req:definitions}, 
\begin{equation}
A\mid X=A\mid D_{\mathrm{t}}\wedge X.\label{eq:add-t-defs}
\end{equation}

\emph{Step 2}. The formulas $Z_{i}$ were constructed such that 
\[
A\wedge X\equiv Z_{1}\vee\cdots\vee Z_{m}
\]
and hence 
\[
D_{\mathrm{t}}\wedge X\models\left(A\leftrightarrow Z_{1}\vee\cdots\vee Z_{m}\right).
\]
R\ref{req:equivalence} then gives 
\begin{equation}
A\mid D_{\mathrm{t}}\wedge X=\left(t_{1}\vee\cdots\vee t_{m}\mid D_{\mathrm{t}}\wedge X\right).\label{eq:rewrite-query}
\end{equation}

\emph{Step 3}. Define the following:
\begin{align*}
I_{j} & =\left\{ i\colon1\leq i\leq n,\,\rho_{i}\left(s_{j}\right)=1\right\} \\
D_{\mathrm{s},j} & =s_{j}\leftrightarrow\bigvee_{i\in I_{j}}t_{i}\\
D_{\mathrm{s}} & =D_{\mathrm{s},1}\wedge\cdots\wedge D_{\mathrm{s},p}.
\end{align*}
Consider how to construct the set of truth assignments $\tilde{\rho}$
on $S\cup T$ that satisfy both $\left\langle t_{1},\ldots,t_{n}\right\rangle $
and $D_{\mathrm{s}}$:
\begin{enumerate}
\item Choose any $i\in\left\{ 1,\ldots,n\right\} $.
\item Set $\tilde{\rho}\left(t_{i}\right)=1$ and $\tilde{\rho}\left(t_{h}\right)=0$
for $h\neq i$.
\item For $j\in\left\{ 1,\ldots,p\right\} $, set $\tilde{\rho}\left(s_{j}\right)$
to the unique value required to satisfy $D_{\mathrm{s},j}$; this
value is $1$ iff $i\in I_{j}$, and $i\in I_{j}$ iff $\rho_{i}\left(s_{j}\right)=1$,
so the required value is just $\rho_{i}\left(s_{j}\right)$.
\end{enumerate}
1 and 2 construct all the ways of ensuring that $\left\langle t_{1},\ldots,t_{n}\right\rangle $
is satisfied, and $3$ then is the only way to finish defining $\tilde{\rho}$
that satisfies $D_{\mathrm{s}}$.

Similarly, consider how to construct the set of truth assignments
$\tilde{\rho}$ on $S\cup T$ that satisfy both $X$ and $D_{\mathrm{t}}$:
\begin{enumerate}
\item Choose any $i\in\left\{ 1,\ldots,n\right\} $. Recall that $\rho_{i}$
is one of the truth assignments satisfying $X$.
\item For $j\in\left\{ 1,\ldots,p\right\} $, set $\tilde{\rho}\left(s_{j}\right)=\rho_{i}\left(s_{j}\right)$.
\item For $h\in\left\{ 1,\ldots,n\right\} $, set $\tilde{\rho}\left(t_{h}\right)$
to the unique value required to satisfy $D_{\mathrm{t},h}$. This
is just $\rho_{i}\left\llbracket Z_{h}\right\rrbracket $, which is
1 for $h=i$ and 0 for $h\neq i$.
\end{enumerate}
1 and 2 construct all the ways of ensuring that $X$ is satisfied,
and 3 then is the only way to finish defining $\tilde{\rho}$ that
satisfies $D_{\mathrm{t}}$.

But these two sets of truth assignments are the same set! Therefore
\[
D_{\mathrm{s}}\wedge\left\langle t_{1},\ldots,t_{n}\right\rangle \equiv D_{\mathrm{t}}\wedge X
\]
and so, by R\ref{req:equivalence}, 
\begin{equation}
\left(t_{1}\vee\cdots\vee t_{m}\mid D_{\mathrm{t}}\wedge X\right)=\left(t_{1}\vee\cdots\vee t_{m}\mid D_{\mathrm{s}}\wedge\left\langle t_{1},\ldots,t_{n}\right\rangle \right).\label{eq:rewrite-premise}
\end{equation}

\emph{Step 4}. Using R\ref{req:definitions} we have 
\begin{equation}
\left(t_{1}\vee\cdots\vee t_{m}\mid D_{\mathrm{s}}\wedge\left\langle t_{1},\ldots,t_{n}\right\rangle \right)=\left(t_{1}\vee\cdots\vee t_{m}\mid\left\langle t_{1},\ldots,t_{n}\right\rangle \right)\label{eq:drop-s-defs}
\end{equation}
since the symbols $s_{1},\ldots,s_{p}$ appear only on the left-hand-sides
of the definitions in $D_{\mathrm{s}}$.

Combining (\ref{eq:add-t-defs})\textendash (\ref{eq:drop-s-defs})
yields the theorem.
\end{proof}

\subsection{Additional consequences}

In light of Lemma \ref{lem:sample-space} we define the following:
\begin{defn}
For any $n>0$ and $0\leq m\leq n$,
\[
\Upsilon_{2}\left(m,n\right)=\left(s_{1}\vee\cdots\vee s_{m}\mid\left\langle s_{1},\ldots,s_{n}\right\rangle \right),
\]
where $s_{1},\ldots,s_{n}\in\mathcal{S}$ are $n$ distinct propositional
symbols.
\end{defn}
We may then restate Lemma \ref{lem:sample-space} as follows:
\begin{cor}
\label{cor:plausibility-from-upsilon2}Let $A\in\Phi\left(S\right)$
and $X\in\Phi^{+}\left(S\right)$ for some finite $S\subseteq\mathcal{S}$.
Then R\ref{req:equivalence} and R\ref{req:definitions} together
imply that
\end{cor}
\[
A\mid X=\Upsilon_{2}\left(\#_{S}\left(A\wedge X\right),\#_{S}\left(X\right)\right).
\]

\begin{proof}
Let $m=\#_{S}\left(A\wedge X\right)$ and $n=\#_{S}\left(X\right)>0$.
Choose any $n$ symbols $t_{1},\ldots,t_{n}\in\mathcal{S}$ disjoint
from both $\sigma\left\llbracket A,X\right\rrbracket $ and the set
of symbols $\left\{ s_{1},\ldots,s_{n}\right\} $ in the definition
of $\Upsilon_{2}$. Then two applications of Lemma \ref{lem:sample-space}
yields
\begin{align*}
A\mid X & =\left(t_{1}\vee\cdots\vee t_{m}\mid\left\langle t_{1},\ldots,t_{n}\right\rangle \right)\\
 & =\left(s_{1}\vee\cdots\vee s_{m}\mid\left\langle s_{1},\ldots,s_{n}\right\rangle \right)\\
 & =\Upsilon_{2}(m,n).
\end{align*}
\end{proof}
We also obtain invariance under change of variables as an immediate
consequence: 
\begin{cor}
\label{cor:transformation-invariance} Let $f$ be a change-of-variables
transformation between $\left(A,X\right)$ and $\left(A',X'\right)$.
Then R\ref{req:equivalence} and R\ref{req:definitions} together
imply that 
\[
A'\mid X'=A\mid X.
\]
\end{cor}
\begin{proof}
Let $f$ map from truth assignments on $S\supseteq\sigma\left\llbracket A,X\right\rrbracket $
to truth assignments on $S'\supseteq\sigma\left\llbracket A',X'\right\rrbracket $.
Then $\#_{S}\left(X\right)=\#_{S'}\left(X'\right)$ and $\#_{S}\left(A\wedge X\right)=\#_{S'}\left(A'\wedge X'\right)$;
the result then follows from Corollary \ref{cor:plausibility-from-upsilon2}.
\end{proof}

\section{Invariance under Addition of Irrelevant Information \label{sec:Invariance-under-addition-of-irrelevant}}

Suppose we are interested in a problem domain whose concepts are represented
by the propositional symbols in some set $S$. A formula $Y$ containing
no symbol from $S$ tells us nothing about this domain; it is irrelevant
information. Adding $Y$ to the premise does not allow us to draw
any new conclusions involving only symbols in $S$.

Specifically, let $Y$ be a satisfiable formula having no symbols
in common with $X$ or $A$, and define $Z=Y\wedge X$. Then $X\models A$
if and only if $Z\models A$, and consequently the implication orderings
$\plle X$ and $\plle Z$ are identical on $\Phi\left(\mathcal{S}\setminus\sigma\left\llbracket Y\right\rrbracket \right)$.
We require the plausibility function to be invariant in the same way:
\begin{requirement}
\label{req:irrelevance}Let $Y$ be a satisfiable formula with $\sigma\left\llbracket X,A\right\rrbracket \cap\sigma\left\llbracket Y\right\rrbracket =\emptyset$.
Then $A\mid X=A\mid Y\wedge X$.
\end{requirement}
Again, one cannot evade the force of this Requirement by supposing
a problem domain with a limited set of symbols, as discussed for R\ref{req:definitions}.
Furthermore, even if we were to associate a finite set of allowable
symbols with each different problem domain, the notion of combining
two unrelated problem domains into one would still make sense. Forbidding
such a combining operation would be an artificial and unreasonable
restriction, as one can already do this in CPL.

\subsection{Independence}

The Requirements so far do not force $A\mid X$ to be any sort of
conditional probability; but \emph{if} $A\mid X$ is the conditional
probability of $A$ given $X$ for some probability distribution,
R\ref{req:irrelevance} implies that we don't come pre-supplied with
dependencies between the atomic propositions. Any such dependencies
have to be created by information in $X$. This is in line with our
intention that the plausibility function be \emph{universal}, one
single function used in all problem domains, computed using no source
of information other than the query and premise themselves, with any
information needed to distinguish different problem domains required
to be included in the premise.

Carnap\textquoteright s proposed confirmation function $c^{*}$ in
particular violates R\ref{req:irrelevance}, as it imposes a probabilistic
dependency between any two atomic sentences having the same predicate
and the same number of distinct arguments. In particular, it is a
violation of R\ref{req:irrelevance} that, for distinct individual
constants $a_{1},\ldots,a_{k+1}$ and monadic predicate $\pi$, we
have
\[
c^{*}\left(\pi\left(a_{k+1}\right)\right)=\frac{1}{2}
\]
but 
\[
c^{*}\left(\pi\left(a_{k+1}\right)\mid\pi\left(a_{1}\right)\wedge\cdots\wedge\pi\left(a_{k}\right)\right)\approx1\mbox{ for large }k.
\]
Carnap finds it necessary to introduce this dependency between atomic
sentences to allow induction, but as we will show in Section \ref{subsec:Uniform-Versus-Non-uniform},
the problem arises only because he omits background information from
the premise / evidence. Once the necessary background information
is included in the premise there is no longer a violation of R\ref{req:irrelevance}.

\subsection{Scale invariance of $\Upsilon_{2}$}

Suppose that $S=\sigma\left\llbracket A,X\right\rrbracket $ and $T$
is obtained from $S$ by adding $r$ symbols not found in $S$. Then
$\#_{T}(X)=2^{r}\cdot\#_{S}(X)$ and $\#_{T}\left(A\wedge X\right)=2^{r}\cdot\#_{S}\left(A\wedge X\right)$,
from which we conclude that $\Upsilon_{2}\left(2^{r}m,2^{r}n\right)=\Upsilon_{2}(m,n)$.
Adding R\ref{req:irrelevance} allows us to extend this scale invariance
to multipliers $k$ that are not powers of 2, and hence to show that
$A\mid X$ is a function only of the \emph{ratio} of $\#_{S}\left(A\wedge X\right)$
to $\#_{S}(X)$.
\begin{lem}
\label{lem:rescaling-upsilon}Suppose that R\ref{req:equivalence},
R\ref{req:definitions}, and R\ref{req:irrelevance} hold. Then for
every $n,k>0$ and $0\leq m\leq n$,
\[
\Upsilon_{2}\left(km,kn\right)=\Upsilon_{2}\left(m,n\right).
\]
\end{lem}
\begin{proof}
Let $S_{1}=\left\{ s_{11},\ldots,s_{1n}\right\} $ and $S_{2}=\left\{ s_{21},\ldots,s_{2k}\right\} $
be two disjoint sets of propositional symbols. There are $kn$ truth
assignments on $S_{1}\cup S_{2}$ satisfying both $\left\langle s_{21},\ldots,s_{2k}\right\rangle $
and $\left\langle s_{11},\ldots,s_{1n}\right\rangle $, and $km$
of these truth assignments also satisfy $s_{11}\vee\cdots\vee s_{1m}$.
Then 
\begin{eqnarray*}
\Upsilon_{2}\left(m,n\right) & = & \left(s_{11}\vee\cdots\vee s_{1m}\mid\left\langle s_{11},\ldots,s_{1n}\right\rangle \right)\\
 & = & \left(s_{11}\vee\cdots\vee s_{1m}\mid\left\langle s_{21},\ldots,s_{2k}\right\rangle \wedge\left\langle s_{11},\ldots,s_{1n}\right\rangle \right)\\
 & = & \Upsilon_{2}\left(km,kn\right).
\end{eqnarray*}
The first and third equalities follows from Corollary \ref{cor:plausibility-from-upsilon2}.
The second equality follows from R\ref{req:irrelevance}, invariance
under addition of irrelevant information.
\end{proof}
\begin{defn}
For any $n>0$ and $0\leq m\leq n$, 
\[
\Upsilon_{1}\left(\frac{m}{n}\right)=\Upsilon_{2}\left(m,n\right).
\]
\end{defn}
Lemma \ref{lem:rescaling-upsilon} ensures that $\Upsilon_{1}\left(r\right)$
is uniquely defined for any rational $r$ in the unit interval when
the appropriate Requirements hold. We then obtain the following:
\begin{cor}
\label{cor:sample-space-upsilon}Let $S\subseteq\mathcal{S}$ be finite,
$A\in\Phi\left(S\right)$, and $X\in\Phi^{+}(S)$. Then R\ref{req:equivalence},
R\ref{req:definitions}, and R\ref{req:irrelevance} together imply
that 
\[
A\mid X=\Upsilon_{1}\left(\frac{\#_{S}\left(A\wedge X\right)}{\#_{S}\left(X\right)}\right).
\]
\end{cor}
\begin{proof}
Let $m=\#_{S}\left(A\wedge X\right)$ and $n=\#_{S}\left(X\right)$.
From Corollary \ref{cor:plausibility-from-upsilon2} and Lemma \ref{lem:rescaling-upsilon}
we get
\[
A\mid X=\Upsilon_{2}\left(m,n\right)=\Upsilon_{1}\left(m/n\right).
\]
\end{proof}

\section{Preservation of Existing Distinctions in Degree of Plausibility \label{sec:Ordering-Properties}}

Our final requirement is that the plausibility function be consistent
with the implication ordering for the premise. Strictly more plausible
queries, according to the implication ordering, must yield strictly
greater plausibility values. 
\begin{requirement}
\label{req:ordering}\label{req:last}There is a partial order $\leq_{\mathbb{P}}$
on $\mathbb{P}$ such that, for any satisfiable formula $X$, if $A\pllt XB$
then $A\mid X<_{\mathbb{P}}B\mid X$.
\end{requirement}
As usual, we understand $p_{1}<_{\mathbb{P}}p_{2}$ to mean $p_{1}\leq_{\mathbb{P}}p_{2}$
and $p_{1}\neq p_{2}$.

The previous Requirements all have the effect of collapsing together
what might otherwise be distinct plausibilities, but say nothing about
when plausibilities must remain distinct from each other. They do
not even rule out the possibility that all plausibilities collapse
down to a single value. Adding R\ref{req:ordering} prevents any further
collapse of plausibility values beyond that of Corollary \ref{cor:sample-space-upsilon},
as we prove with Lemma \ref{lem:upsilon-is-strictly-increasing} below.

\subsection{A too-simple plausibility function}

Suppose that we choose the plausibility function to be a direct translation
of the logical consequence relation, defining
\[
A\mid X=\begin{cases}
\mathsf{F} & \mbox{if }X\models\neg A\\
\mathsf{T} & \mbox{if }X\models A\\
\mathsf{u} & \mbox{otherwise}
\end{cases}
\]
with $\mathbb{P}=\left\{ \mathsf{F},\mathsf{u},\mathsf{T}\right\} $
and $\mathsf{F}<_{\mathbb{P}}\mathsf{u}<_{\mathbb{P}}\mathsf{T}$.
It is straightforward to verify that this definition satisfies R\ref{req:equivalence},
R\ref{req:definitions}, and R\ref{req:irrelevance}, by appealing
to the corresponding properties of the logical equivalence relation
$\models$. However, this definition violates R\ref{req:ordering},
since R\ref{req:ordering} implies that $\mathbb{P}$ must be infinite.

To see this, suppose that $\mathbb{P}$ is finite. Choose $n>\left|\mathbb{P}\right|$
distinct symbols $s_{1},\ldots,s_{n}$ and define 
\[
X=\bigwedge_{i=1}^{n-1}\left(s_{i}\rightarrow s_{i+1}\right).
\]
Then 
\[
s_{1}\pllt Xs_{2}\pllt X\cdots\pllt Xs_{n}
\]
and R\ref{req:ordering} therefore mandates that 
\[
s_{1}\mid X<s_{2}\mid X<\cdots<s_{n}\mid X;
\]
but this cannot be, as $\mathbb{P}$ contains fewer than $n$ elements.

\subsection{Probability from plausibility}

We proceed with the proof of our main result, starting with a lemma.
\begin{lem}
\label{lem:upsilon-is-strictly-increasing}If R1\textendash R\ref{req:last}
hold then for all $r,r'\in\mathbb{Q}_{01}\triangleq\mathbb{Q}\cap[0,1]$
we have
\begin{align*}
\Upsilon_{1}\left(r\right)=\Upsilon_{1}\left(r'\right) & \Leftrightarrow r=r'\\
\Upsilon_{1}\left(r\right)<_{\mathbb{P}}\Upsilon\left(r'\right) & \Leftrightarrow r<r'.
\end{align*}
\end{lem}
\begin{proof}
We may express $r$ and $r'$ as ratios with a common denominator
$n>0$ as $r=m/n$ and $r'=m'/n$. Using R\ref{req:ordering} and
writing $X$ for $\left\langle s_{1},\ldots,s_{n}\right\rangle $
we have
\begin{align}
r<r' & \Rightarrow m<m'\nonumber \\
 & \Rightarrow\left(s_{1}\vee\cdots\vee s_{m}\right)\pllt X\left(s_{1}\vee\cdots\vee s_{m'}\right)\nonumber \\
 & \Rightarrow\Upsilon_{1}(r)<_{\mathbb{P}}\Upsilon_{1}(r').\label{eq:upsilon-1-is-strict}
\end{align}
Furthermore, using antisymmetry of the partial order $\leq_{\mathbb{P}}$
and (\ref{eq:upsilon-1-is-strict}),
\begin{align*}
r\not<r' & \Rightarrow\left(r=r'\right)\vee\left(r'<r\right)\\
 & \Rightarrow\Upsilon_{1}\left(r'\right)\leq_{\mathbb{P}}\Upsilon_{1}\left(r\right)\\
 & \Rightarrow\Upsilon_{1}\left(r\right)\not<_{\mathbb{P}}\Upsilon_{1}\left(r'\right).
\end{align*}
Trivially,
\[
r=r'\Rightarrow\Upsilon_{1}\left(r\right)=\Upsilon_{1}\left(r'\right).
\]
Furthermore, using (\ref{eq:upsilon-1-is-strict}) again,
\begin{align*}
r\neq r' & \Rightarrow\left(r<r'\right)\vee\left(r'<r\right)\\
 & \Rightarrow\Upsilon_{1}\left(r\right)\neq\Upsilon_{1}\left(r'\right).
\end{align*}
\end{proof}
And now we arrive at the central result of this paper.
\begin{thm}
\label{thm:classical-prob-def}If R1\textendash R\ref{req:last} hold
then
\begin{enumerate}
\item $\Upsilon_{1}$ is an order isomorphism between the posets $\left(\mathbb{Q}_{01},\leq\right)$
and $\left(\mathbb{P},\leq_{\mathbb{P}}\right)$;
\item for all finite $S\subseteq\mathcal{S}$, $A\in\Phi(S)$, and $X\in\Phi^{+}(S)$
we have 
\[
P\left(A\mid X\right)=\frac{\#_{S}\left(A\wedge X\right)}{\#_{S}(X)}.
\]

where $P=\Upsilon_{1}^{-1}$.

\end{enumerate}
\end{thm}
\begin{proof}
By Corollary \ref{cor:sample-space-upsilon}, $\Upsilon_{1}\colon\mathbb{Q}_{01}\rightarrow\mathbb{P}$
is onto, and by Lemma \ref{lem:upsilon-is-strictly-increasing}, $\Upsilon_{1}$
is a strictly increasing function (hence also one-to-one). So $\Upsilon_{1}$
is an order-preserving bijection between $\mathbb{Q}_{01}$ and $\mathbb{P}$,
that is, it is an order isomorphism.

Since $\Upsilon_{1}$ is a bijection, its inverse $P$ exists. The
second claim is then just a restatement of Corollary \ref{cor:sample-space-upsilon}.
\end{proof}
The laws of probability follow directly from Theorem \ref{thm:classical-prob-def}.
In stating them it is convenient to extend the plausibility function
to unsatisfiable premises using the convention that $A\mid X=\Upsilon_{1}(1)$,
and hence $P\left(A\mid X\right)=1$, when $X$ is unsatisfiable.
This may be justified by noting that for satisfiable $X$ we have
$P\left(A\mid X\right)=1$ whenever $X\models A$, and when $X$ is
unsatisfiable we have $X\models A$ for \emph{all} formulas $A$.
\begin{cor}
\label{cor:laws-of-prob}If R1\textendash R\ref{req:last} hold and
we define $A\mid X=\Upsilon_{1}(1)$ for unsatisfiable $X$, then

\begin{enumerate}
\item \label{enu:prob-law-a}$0\leq P\left(A\mid X\right)\leq1$.
\item \label{enu:prob-known-true}$P\left(A\mid X\right)=1$ if $X\models A$.
\item $P\left(A\mid X\right)=0$ if $X\models\neg A$ and $X$ is satisfiable.
\item \label{enu:prob-law-b}$P\left(\neg A\mid X\right)=1-P\left(A\mid X\right)$
if $X$ is satisfiable.
\item \label{enu:product-rule}$P\left(A\wedge B\mid X\right)=P\left(B\mid X\right)\cdot P\left(A\mid B\wedge X\right)$.
\end{enumerate}
\end{cor}
\begin{proof}
(\ref{enu:prob-law-a})\textendash (\ref{enu:prob-law-b}) are trivial,
but (\ref{enu:product-rule}) merits comment because care must be
taken in handling unsatisfiable premises. There are three cases: 

\begin{enumerate}
\item If $X$ is unsatisfiable then so is $B\wedge X$, and the claim reduces
to $1=1\cdot1$.
\item If $X$ is satisfiable but $B\wedge X$ is not then $X$ logically
implies both $\neg B$ and $\neg\left(A\wedge B\right)$, so $P\left(B\mid X\right)=0$
and $P\left(A\wedge B\mid X\right)=0$ and $P\left(A\mid B\wedge X\right)=1$,
and the claim reduces to $0=0\cdot1$.
\item If $X$ and $B\wedge X$ are both satisfiable, let $S=\sigma\left\llbracket A,B,X\right\rrbracket $,
$n=\#_{S}(X)>0$, $p=\#_{S}\left(B\wedge X\right)>0$, and $m=\#_{S}\left(A\wedge B\wedge X\right)$;
then 
\[
P\left(A\wedge B\mid X\right)=\frac{m}{n}=\frac{p}{n}\cdot\frac{m}{p}=P\left(B\mid X\right)\cdot P\left(A\mid B\wedge X\right).
\]
\end{enumerate}
\end{proof}
Theorem \ref{thm:classical-prob-def} and Corollary \ref{cor:laws-of-prob}
tell us that plausibilities are essentially just probabilities, following
the classical definition of probability as the ratio of favorable
cases to all cases. That is, our four Requirements, based entirely
on preserving existing properties of CPL, lead us to identify finite-set
probability theory as the \emph{uniquely determined} extension of
CPL to a logic of plausible reasoning.

\section{Consistency of Requirements\label{sec:Consistency-of-Requirements}}

An issue that must be addressed for any axiomatic development is whether
its content is vacuous by virtue of there not existing \emph{any}
mathematical structure satisfying the given axioms. If our Requirements
are inconsistent\textemdash if there does not exist any plausibility
function $\left(\cdot\mid\cdot\right)$ for which the Requirements
all hold\textemdash then Theorem \ref{thm:classical-prob-def} is
trivially true, and our entire exercise is pointless. We now show
that this is not the case, by exhibiting a specific plausibility function
that satisfies all the Requirements.

Theorem \ref{thm:classical-prob-def} provides an obvious candidate
for this plausibility function. However, that theorem (and the results
leading up to it) cannot help in \emph{proving} that the Requirements
can be satisfied, as they are consequences of \emph{assuming} that
one already has some plausibility function satisfying the Requirements.
\begin{thm}
R1\textendash R\ref{req:last} are consistent. In particular, suppose
that for any formula $A$ and satisfiable formula $X$ we define 
\[
A\mid X=\frac{\#_{T}\left(A\wedge X\right)}{\#_{T}\left(X\right)},
\]
where $T=\sigma\left\llbracket A,X\right\rrbracket $; then R1\textendash R\ref{req:last}
all hold.
\end{thm}
\begin{proof}
Consider any finite set of symbols $S\supseteq T$. If $S$ contains
$k$ additional symbols beyond those in $T$ then $\#_{S}(X)=2^{k}\#_{T}(X)$
and $\#_{S}\left(A\wedge X\right)=2^{k}\#_{T}\left(A\wedge X\right)$,
hence 
\[
A\mid X=\frac{\#_{S}\left(A\wedge X\right)}{\#_{S}\left(X\right)}.
\]
Thus we may use any superset of the symbols appearing in $A$ and
$X$ when evaluating $A\mid X$. 

We now consider each of the Requirements in turn.

\emph{R\ref{req:equivalence}}. Let $X\equiv Y$ and $A\pleqv XB$
and $S=\sigma\left\llbracket A,B,X,Y\right\rrbracket $. Then $\#_{S}(X)=\#_{S}(Y)$
and $\#_{S}\left(A\wedge X\right)=\#_{S}\left(B\wedge X\right)=\#_{S}\left(B\wedge Y\right)$,
hence $A\mid X=B\mid Y$.

\emph{R\ref{req:definitions}}. Let $Y$ be $\left(s\leftrightarrow E\right)\wedge X$,
where $s$ is a propositional symbol not in $S=\sigma\left\llbracket A,X,E\right\rrbracket $.
Let $S'=S\cup\{s\}$. Each truth assignment on $S$ satisfying $X$
can be extended to a truth assignment on $S'$ satisfying $Y$ in
exactly one way, therefore $\#_{S'}(Y)=\#_{S}(X)$. Likewise, $\#_{S'}\left(A\wedge Y\right)=\#_{S}\left(A\wedge X\right)$.
Hence $A\mid Y=A\mid X$.

\emph{R\ref{req:irrelevance}}. Let $S=\sigma\left\llbracket A,X\right\rrbracket $,
$S'=\sigma\left\llbracket Y\right\rrbracket $, and $T=S\cup S'$.
Since $S$ and $S'$ are disjoint, we have 
\begin{align*}
\#_{T}\left(Y\wedge X\right) & =\#_{S'}(Y)\#_{S}(X)\\
\#_{T}\left(A\wedge Y\wedge X\right) & =\#_{S'}(Y)\#_{S}\left(A\wedge X\right)
\end{align*}
and hence $A\mid X=A\mid Y\wedge X$.

\emph{R\ref{req:ordering}}. Choose $\left(\mathbb{P},\leq_{\mathbb{P}}\right)$
to be $\left(\mathbb{Q}_{01},\leq\right)$. Suppose that $X$ is satisfiable
and let $S=\sigma\left\llbracket A,B,X\right\rrbracket $. If $A\pllt XB$
then all truth assignments satisfying both $A$ and $X$ also satisfy
$B$, and there is some truth assignment satisfying both $B$ and
$X$ that does \emph{not} satisfy $A$. Hence $\#_{S}\left(A\wedge X\right)<\#_{S}\left(B\wedge X\right)$,
yielding $A\mid X<B\mid X$.
\end{proof}

\section{Discussion\label{sec:Discussion}}

\subsection{The classical definition of probability}

The classical definition of probability goes back to Cardano in the
mid 16th Century \cite[Chapter 14]{cardano}; perhaps its clearest
statement was given by Laplace \cite{laplace}:
\begin{quote}
The probability of an event is the ratio of the number of cases favorable
to it, to the number of possible cases, when there is nothing to make
us believe that one case should occur rather than any other, so that
these cases are, for us, equally possible.
\end{quote}
This definition fell out of favor with the rise of both frequentist
and subjective interpretations of probability. Theorem \ref{thm:classical-prob-def}
takes us back to the beginnings of probability theory, validating
the classical definition and sharpening it. We can now say that a
``possible case'' is simply a truth assignment satisfying the premise
$X$. The phrase ``these cases are, for us, equally possible,''
which arguably makes the definition circular, may simply be dropped
as unnecessary. The phrase ``there is nothing to make us believe
that one case should occur rather than any other'' means that we
possess no additional information that, if conjoined with our premise,
would expand the satisfying truth assignments by differing multiplicities. 

We shall illustrate this subtle but important point with Bertrand's
``Box Paradox'' \cite{bertrand-box}. There are three identical
boxes in a row, each with two drawers. One of the boxes, call it GG,
has gold coins in both drawers; one box, call it SS, has silver coins
in both drawers; and the remaining box, call it GS, has a gold coin
in one drawer and a silver coin in the other. Not knowing which is
which, you open the first drawer of the second box, and observe that
it contains a gold coin; what is the probability that the other drawer
also holds a gold coin?

\begin{table}

\begin{centering}
\begin{tabular}{|c|c||c|c||c|c|}
\hline 
D11 & D12 & D21 & D22 & D31 & D32\tabularnewline
\hline 
\hline 
G & G & G & S & S & S\tabularnewline
\hline 
S & S & G & S & G & G\tabularnewline
\hline 
G & S & G & G & S & S\tabularnewline
\hline 
S & G & G & G & S & S\tabularnewline
\hline 
S & S & G & G & G & S\tabularnewline
\hline 
S & S & G & G & S & G\tabularnewline
\hline 
\end{tabular}
\par\end{centering}
\caption{\label{tab:bertrand-box}Tabulation of all possible cases in Bertrand's
``Box Paradox.'' D\emph{ij} means drawer $j$ of box $i$.}
\end{table}

This problem is often resolved by appeal to Bayes' Rule, yielding
a probability of 2/3. Let's apply Theorem \ref{thm:classical-prob-def}
instead. A naïve analysis, using only information about the second
box itself, gives a probability of 1/2: the second box must be either
GG or GS (two cases), and since the first drawer contains a gold coin,
the second drawer also contains a gold coin only if the second box
is GG (one case). But this ignores the (seemingly irrelevant) information
we have about the first and third boxes. Table \ref{tab:bertrand-box}
gives an exhaustive list of all possible cases when the other boxes
are included. We see that the case ``second box is GS'' gets expanded
into two cases, while the case ``second box is GG'' gets expanded
into \emph{four} cases, thereby invalidating the naïve analysis. Using
the expanded table gives the correct answer of 2/3.

\subsection{Uniform versus non-uniform probabilities\label{subsec:Uniform-Versus-Non-uniform}}

One concern about Theorem \ref{thm:classical-prob-def} may be that
it \emph{mandates} the uniform distribution on $\Omega$, the induced
sample space of truth assignments satisfying the premise $X$. But
what other reasonable option is there? Remember that the premise $X$
contains \emph{all} the information to which we have access in determining
our probability distribution. There is no implicit third argument
to the plausibility function that varies from one problem domain to
another. $\Omega$ is just the reification of $X$ as a set\textemdash $X$
tells us that \emph{one} of the elements of $\Omega$ is the correct
description of the situation, and that is \emph{all} it tells us.
It gives us no information by which we could favor one of these possibilities
over another.

Yet non-uniform distributions are the norm in practical applications
of probability theory, and one may ask where they come from. The ``Box
Paradox'' example illustrates one answer: via marginalization. A
uniform distribution at the finest level of granularity can correspond
to a \emph{nonuniform} distribution at coarser levels obtained by
considering the induced sample space for some\emph{ subset} of the
symbols in $\sigma\left\llbracket A,X\right\rrbracket $.

For a more complete answer, let's consider Carnap's objection to the
uniform distribution. He defines a confirmation function $c^{\dagger}$
for $\mathfrak{L}_{N}$ based on a uniform distribution over state-descriptions,
and notes that if $a_{1},\ldots,a_{k},a_{k+1}$ are distinct individual
constants and $\pi$ is a monadic predicate (property), then 
\[
c^{\dagger}\left(\pi\left(a_{k+1}\right),\,\pi\left(a_{1}\right)\wedge\cdots\wedge\pi\left(a_{k}\right)\right)=\frac{1}{2}
\]
for any $k<N$. He concludes \cite[p. 565]{carnap},
\begin{quote}
Thus the choice of $c^{\dagger}$ as the degree of confirmation would
be tantamount to the principle never to let our past experiences influence
our expectations for the future.
\end{quote}
Yet Carnap encounters this problem with $c^{\dagger}$ for precisely
the same reason that he cannot find a uniquely determined confirmation
function. As Jaynes writes \cite[p. 279]{ptlos},
\begin{quote}
Carnap was seeking the general inductive rule (i.e., the rule by which,
given the record of past results, one can make the best possible prediction
of future ones). But\ldots{} he never rises to the level of seeing
that different inductive rules correspond to \emph{different prior
information}. It seems to us obvious\ldots{} that this is the primary
fact controlling induction, without which the problem cannot even
be stated, much less solved; there is no `general inductive rule.'
Yet neither the term `prior information' nor the concept ever appears
in Carnap's exposition.
\end{quote}
This prior information belongs in the premise, and Carnap chooses
not to include it there, as discussed in Section \ref{sec:Carnap-logical-probability}
and Section \ref{sec:The-Plausibility-Function}.

As an example of such prior information, consider Carnap's proposed
confirmation function $c^{*}$ and associated measure function $m^{*}$,
for a language having a single monadic predicate $\pi$. Let us write
$x_{i}$ for the atomic sentence $\pi\left(a_{i}\right)$, where $a_{i}$
is the $i$-th individual constant. Then $m^{*}$ is equivalent to
defining the joint distribution 
\begin{align*}
\theta & \sim\mathrm{Uniform}(0,1)\\
x_{i} & \sim\mathrm{Bernoulli}(\theta)\quad\mbox{independently for all }i
\end{align*}
and marginalizing out $\theta$. That is, give $\theta$ a uniform
distribution over the interval $(0,1)$, then independently give each
$x_{i}$ a probability $\theta$ of being true.

We now construct a propositional formula that expresses an arbitarily
close approximation of this prior information. Let $I$ and $K$ be
large, positive integers. Consider the $x_{i}$, $1\leq i\leq I$,
as propositional symbols. Let $h_{k}$, $0\leq k\leq K$, have the
intended interpretation ``$\theta=k/K."$ Let us imagine that individual
$i$ may be in any of $K$ distinct fine-grained states, and let $s_{ij}$,
$1\leq j\leq K$, have the intended interpretation that individual
$i$ is in state $j$. Finally, define $X$ to be the conjunction
of the following $(K^{2}+K+1)I+1$ formulas: 
\begin{flalign*}
\left\langle h_{0},\ldots,h_{K}\right\rangle \\
\left\langle s_{i1},\ldots,s_{iK}\right\rangle  & \mbox{ for }1\leq i\leq I\\
h_{k}\wedge s_{ij}\rightarrow l_{ijk} & \mbox{ for }1\leq i\leq I\mbox{, }1\leq j\leq K\mbox{, }0\leq k\leq K
\end{flalign*}
where 
\begin{align*}
l_{ijk} & =\begin{cases}
x_{i} & \mbox{if }j\leq k\\
\neg x_{i} & \mbox{if }j>k.
\end{cases}
\end{align*}
That is, exactly one of the $h_{k}$ is true; for each $i$, exactly
one of the $s_{ij}$ is true; and if $h_{k}$ is true then each $x_{i}$
is true in $k$ out of the $K$ possible states for individual $i$.
Using Theorem \ref{thm:classical-prob-def} we then have 
\begin{align*}
P\left(x_{i}\mid h_{k}\wedge X\right) & =k/K\quad\mbox{for all }i\mbox{ independently}\\
P\left(h_{k}\mid X\right) & =1/\left(k+1\right).
\end{align*}
This illustrates the general lesson: non-uniform probabilities arise
by introducing latent variables, including in the premise information
that links the latent variables to observables, and then marginalizing
out latent variables.

\subsection{Infinite domains\label{subsec:Infinite-Domains}}

How might one extend these results to infinite domains, which are
required for the bulk of practical applications of probability theory?
Jaynes proposes a \emph{finite sets policy} \cite[p. 43]{ptlos}:
\begin{quote}
It is very important to note that our consistency theorems have been
established only for probabilities assigned on \emph{finite sets}
of propositions. In principle, every problem must start with such
finite-set probabilities; extension to infinite sets is permitted
only when this is the result of a well-defined and well-behaved limiting
process from a finite set.
\end{quote}
In the same vein, he writes \cite[p. 663]{ptlos},
\begin{quote}
In probability theory, it appears that the only safe procedure known
at present is to derive our results first by strict application of
the rules of probability theory on finite sets of propositions; then,
after the finite-set result is before us, observe how it behaves as
the number of propositions increases indefinitely.
\end{quote}
As an example, consider $P\left(y<\left(1-x\right)^{2}\mid x,y\in[0,1)\wedge y<x^{2}\right)$.
We can consider this to be the limiting value of $P\left(A_{n}\mid X_{n}\right)$
as $n\rightarrow\infty$, where the queries $A_{n}$ and premises
$X_{n}$ are defined as follows:
\begin{enumerate}
\item Symbols $a_{i}$ and $b_{i}$, for $1\leq i\leq n$, are intended
to mean $i-1\leq nx<i$ and $i-1\leq ny<i$ respectively. 
\item Let $A_{n}$ be $\bigvee_{(i,j)\in K}\left(a_{i}\wedge b_{j}\right)$
where $K=\left\{ (i,j)\colon j/n\leq\left(1-i/n\right)^{2}\right\} $. 
\item Let $X_{n}$ be $\bigvee_{(i,j)\in L}\left(a_{i}\wedge b_{j}\right)$
where $L=\left\{ (i,j)\colon j/n\leq\left(i/n\right)^{2}\right\} $.
\end{enumerate}
Figure \ref{fig:region-under-curve} illustrates $X_{n}\wedge A_{n}$
in black and $X_{n}\wedge\neg A_{n}$ in gray for $n=30$. As $n\rightarrow\infty$,
$A_{n}$ tends in the limit to the desired query ``$y<\left(1-x\right)^{2}$''
and $X_{n}$ tends in the limit to the desired premise ``$y<x^{2}$,''
with $x,y\in[0,1)$ implicit in the problem encoding. 

\begin{figure}
\begin{centering}
\includegraphics[scale=0.5]{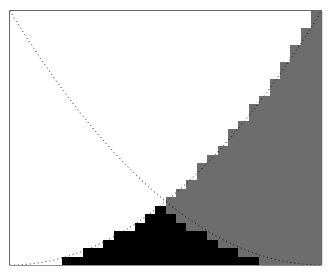}
\par\end{centering}
\caption{\label{fig:region-under-curve}Approximating $P\left(y<\left(1-x\right)^{2}\mid y<x^{2}\right)$
with $n=30$}
\end{figure}

Though straightforward, it would be tedious to have to explicitly
construct the limiting process and find the limiting value every time
we considered a probability involving an infinite domain. Modern probability
theory is based on measure theory, so it should be no surprise that
measure theory provides the tools to automate this process of constructing
a sequence of finite approximations that converge to a limit. We will
not attempt to provide a full account of this large topic. By way
of illustration, however, we will discuss one particularly simple
case. We only sketch things out here; see Appendix \ref{sec:Measure-Theory}
for more details.

Consider the Cantor set of infinite binary sequences $\mathbb{B}^{\omega}$,
where $\mathbb{B}=\left\{ 0,1\right\} $. Enumerating the elements
of $\mathcal{S}$ as $s_{1},s_{2},\ldots$, we can consider $\mathbb{B}^{\omega}$
to be the set of truth assignments on $\mathcal{S}$ if we identify
a truth assignment $\rho$ with the infinite sequence $w$ such that
$w_{i}=\rho\left(s_{i}\right)$ for all $i$. We write $\mathcal{C}$
and $\mu_{\mathcal{C}}$ for the Borel $\sigma$-algebra on $\mathbb{B}^{\omega}$
and Borel measure on $\mathcal{C}$ respectively, which may be considered
a uniform distribution over $\mathbb{B}^{\omega}$. Defining 
\[
[A]=\left\{ w\in\mathbb{B}^{\omega}\colon w\mbox{ satisfies }A\right\} 
\]
for any $A\in\Phi\left(\mathcal{S}\right)$, and 
\[
\Pr\left(\tilde{A};\tilde{X},\mu\right)=\frac{\mu\left(\tilde{A}\cap\tilde{X}\right)}{\mu\left(\tilde{X}\right)}
\]
for any two measurable sets $\tilde{A},\tilde{X}$ and measure $\mu$
with $\mu\left(\tilde{X}\right)>0$, we find that 
\[
P\left(A\mid X\right)=\Pr\left([A];[X],\mu_{\mathcal{C}}\right).
\]
Finally, Theorem \ref{thm:well-defined-limit} (Appendix \ref{sec:Measure-Theory})
states that for any measurable sets $\tilde{A},\tilde{X}\in\mathcal{C}$
with $\mu_{\mathcal{C}}\left(\tilde{X}\right)>0$ there exists a sequence
of formulas $A_{i}\in\Phi\left(\mathcal{S}\right)$ and $X_{i}\in\Phi^{+}\left(\mathcal{S}\right)$
such that 
\begin{gather*}
\mu_{\mathcal{C}}\left(\left[A_{i}\right]\triangle\tilde{A}\right)\rightarrow0\\
\mu_{\mathcal{C}}\left(\left[X_{i}\right]\triangle\tilde{X}\right)\rightarrow0\\
P\left(A_{i}\mid X_{i}\right)\rightarrow\Pr\left(\tilde{A};\tilde{X},\mu_{\mathcal{C}}\right)
\end{gather*}
as $i\rightarrow\infty$, where $\triangle$ stands for set difference.
That is, Jaynes's ``well-defined and well-behaved limiting process''
is guaranteed to exist under the conditions of the Theorem, and $\Pr\left(\tilde{A};\tilde{X},\mu_{\mathcal{C}}\right)$
is the limiting probability.

Now we turn to the space $\Omega=\mathbb{B}^{m}\times[0,1)^{n}$.
We write $\mathcal{D}$ for the powerset of $\mathbb{B}^{m}$ (the
maximal $\sigma$-algebra on $\mathbb{B}^{m}$) and define $\mu_{\mathcal{D}}(\tilde{A})=\left|\tilde{A}\right|/2^{m}$
for $\tilde{A}\in\mathcal{D}$. We write $\mathcal{B}$ and $\mu_{\mathcal{B}}$
for the Borel $\sigma$-algebra on $[0,1)^{n}$ and Borel measure
on $\mathcal{B}$, respectively. Let $\mathcal{A}=\sigma\left(\mathcal{D}\times\mathcal{B}\right)$
be the product $\sigma$-algebra of $\mathcal{D}$ and $\mathcal{B}$,
and let $\mu_{\mathcal{A}}=\mu_{\mathcal{D}}\times\mu_{\mathrm{\mathcal{B}}}$
be the product measure of $\mu_{\mathcal{D}}$ and $\mu_{\mathcal{B}}$.
The measure $\mu_{\mathcal{A}}$ may be considered a uniform distribution
over $\mathbb{B}^{m}\times[0,1)^{n}$, with $\mu_{\mathcal{A}}\left(\tilde{A}\times\tilde{B}\right)$
being $\left|\tilde{A}\right|/2^{m}$ times the $n$-dimensional hypervolume
of $\tilde{B}$.

The set $\mathbb{B}^{*}1^{\omega}$ of binary sequences ending in
an infinite sequence of 1's is a measurable set of measure 0. Let
us define $\mathbb{I}=\mathbb{B}^{\omega}\setminus\mathbb{B}^{*}1^{\omega}$
to be all infinite binary sequences \emph{except} this measure-0 set.
Define the function $f\colon\mathbb{I}\rightarrow\Omega$ as
\begin{align*}
f\left(w\right) & =\left(f_{0}(w),r\left(f_{1}(w)\right),\ldots,r\left(f_{n}(w)\right)\right)\\
f_{0}(w) & =w_{1}\cdots w_{m}\\
f_{j}(w) & =v_{1}v_{2}\cdots\mbox{ where }v_{i}=w_{m+\iota(i,j)},\mbox{ for }j\neq0\\
r(v) & =\sum_{i=1}^{\infty}2^{-i}v_{i}\\
\iota(i,j) & =j+(i-1)n.
\end{align*}
That is, applying the mapping $f$ amounts to interpreting symbol
$s_{m+\iota(i,j)}$, for $i\geq1$ and $1\leq j\leq n$, as the $i$-th
bit in the infinite binary expansion of $x_{j}\in[0,1)$, or more
precisely, as the proposition ``$\left\lfloor 2^{i}x_{j}\right\rfloor \bmod2=1$.''
We interleave $n$ infinite sequences into one sequence by mapping
index $i$ of sequence $j$ to index $\iota(i,j)$ of the combined
sequence. Using symbols $s_{m+1}$ through $s_{m+\iota(k,n)}$ we
can express any subspace of $[0,1)^{n}$ at a granularity of hypercubes
of length $2^{-k}$ on each side, and we can make this granularity
as fine as desired by choosing $k$ sufficiently large.

The function $f$ is a bijection between $\mathbb{I}$ and $\Omega$.
(We excluded $\mathbb{B}^{*}1^{\omega}$ to ensure this, as dyadic
rationals $m/2^{n}$ have two possible binary expansions.) Furthermore,
both $f$ and $f^{-1}$ are measurable: $f^{-1}(A)\in\mathcal{C}$
and $f^{-1}(A)\subseteq\mathbb{I}$ whenever $A\subseteq\mathcal{A}$,
and $f(B)\in\mathcal{A}$ whenever $B\in\mathcal{C}$ and $B\subseteq\mathbb{I}$.
Finally, $f$ is measure-preserving: $\mu_{\mathcal{C}}\left(f^{-1}(A)\right)=\mu_{\mathcal{A}}\left(A\right)$
whenever $A\in\mathcal{A}$. This guarantees that 
\[
\Pr\left(\tilde{A};\tilde{X},\mu_{A}\right)=\Pr\left(f^{-1}\left(\tilde{A}\right);f^{-1}\left(\tilde{X}\right),\mu_{\mathcal{C}}\right).
\]
Therefore, we can apply Theorem \ref{thm:well-defined-limit} and
find that for any measurable sets $\tilde{A},\tilde{X}\in\mathcal{A}$
with $\mu_{\mathcal{A}}\left(\tilde{X}\right)>0$ there exist sequences
of formulas $A_{i}$ and $X_{i}$, with $X_{i}$ satisfiable, such
that
\begin{gather*}
\mu_{\mathcal{C}}\left(\left[A_{i}\right]\triangle f^{-1}\left(\tilde{A}\right)\right)\rightarrow0\\
\mu_{\mathcal{C}}\left(\left[X_{i}\right]\triangle f^{-1}\left(\tilde{X}\right)\right)\rightarrow0\\
P\left(A_{i}\mid X_{i}\right)\rightarrow\Pr\left(\tilde{A};\tilde{X},\mu_{\mathcal{A}}\right)
\end{gather*}
as $i\rightarrow\infty$.

In the example given at the beginning of this section, we have $m=0$,
$n=2$, and 
\begin{align*}
\tilde{A} & =\left\{ (x,y)\in[0,1)^{2}\colon y<(1-x)^{2}\right\} \\
\tilde{X} & =\left\{ (x,y)\in[0,1)^{2}\colon y<x^{2}\right\} .
\end{align*}
The above results tell us that we don't need to explicitly construct
the sequence of approximating formulas for this example; it is guaranteed
to exist, and the limiting probability is 
\[
\Pr\left(\tilde{A};\tilde{X},\mu_{\mathcal{A}}\right)=\frac{\int_{0}^{1}\min\left(x^{2},\left(1-x\right)^{2}\right)\mathrm{d}x}{\int_{0}^{1}x^{2}\mathrm{d}x}=\frac{1}{4}.
\]

As another example, let us revisit and generalize the inductive model
described in Section \ref{subsec:Uniform-Versus-Non-uniform}:
\begin{align*}
\theta & \sim\mathrm{Distr}\left(F\right)\\
x_{i} & \sim\mathrm{Bernoulli}(\theta)\quad\mbox{independently for all }1\le i\leq I
\end{align*}
where $\mathrm{Distr}(F)$ is the distribution on the unit interval
with cdf $F$, which we take to be continuous (and hence invertible).
Doing a change of variables and augmenting with latent variables $s_{i}$,
the above is equivalent to 
\begin{align*}
p & \sim\mathrm{Uniform}(0,1)\\
\theta & =F^{-1}(p)\\
s_{i} & \sim\mathrm{Uniform}(0,1)\\
x_{i} & =\begin{cases}
1 & \mbox{if }s_{i}<\theta\\
0 & \mbox{otherwise}
\end{cases}
\end{align*}
after marginalizing out $p$ and $s$. We have independent uniform
distributions on $p$ and each $s_{i}$, plus equations relating each
$x_{i}$ to $p$ and $s_{i}$; hence the above is equivalent to using
as premise the measurable set 
\[
\tilde{X}=\left\{ \left(x,p,s\right)\in\mathbb{B}^{I}\times[0,1)^{1+I}\colon x_{i}=1\Leftrightarrow s_{i}<F^{-1}\left(p\right)\mbox{ for all }1\leq i\leq I\right\} ,
\]
again using the $\sigma$-algebra $\mathcal{A}$ and measure $\mu_{\mathcal{A}}$,
with $m=I$ and $n=I+1$. For any measurable $\tilde{A}$ we are again
guaranteed that $\Pr\left(\tilde{A};\tilde{X},\mu_{\mathcal{A}}\right)$
is the limiting probability obtained from a sequence of approximating
formulas $A_{i}$ and $X_{i}$.

This is not a complete solution to handling infinite problem domains.
For instance, in the example above we used $\mathbb{B}^{I}$ (with
finite $I$) instead of $\mathbb{B}^{\omega}$, because $\mu_{\mathcal{A}}\left(\tilde{X}\right)\rightarrow0$
as $I\rightarrow\infty$. In addition, the measure $\mu_{\mathcal{C}}$
on $\mathcal{C}$ and encoding of real numbers used above works well
for a bounded interval like $[0,1)$ but does not suffice for unbounded
intervals, such as all of $\mathbb{R}$. For such cases we need alternative
measures on $\mathcal{C}$ and corresponding analogs to Theorem \ref{thm:well-defined-limit},
along with alternative encodings for these domains. 

Other work that remains to be done on this topic includes the following:
\begin{enumerate}
\item Finding a general method of constructing the needed sequence of approximating
formulas for \emph{any} computable probability measure \cite{edalat-computable-measure-theory,weihrauch-tavana}.
\item Extending our language of propositional formulas to express measurable
sets beyond just the cylinder sets, while ensuring that $P\left(A\mid X\right)$
remains computable, as is appropriate for a logical system.
\end{enumerate}
We have made some initial investigations of these open issues and
believe that they can be resolved.

\section{Conclusion}

We have strengthened the case for probability theory as the uniquely
determined extension of classical propositional logic to a logic of
plausible reasoning. Our proof relies on a small and simple set of
requirements such a logic must satisfy. These requirements are harder
to dispute than those of previous such efforts because every one of
the requirements is motivated by a desire to retain in our extended
logic some property of CPL. A crucial distinction between our approach
and similar previous work is that $A\mid X$ depends \emph{only }on
the explicit arguments $A$ and $X$, and not on any other domain-specific
or problem-specific information; any such relevant information must
be included in the premise $X$. This makes the plausibility function
a legitimate analog of the logical consequence relation: the truth
or falsity of $X\models A$ likewise depends only on $X$ and $A$,
and not on any implicit domain-specific or problem-specific information.

R2 (invariance under definition of new symbols) in conjunction with
R1 (logical equivalence) turns out to have far-reaching implications.
It yields invariance under renaming of propositional symbols and,
in fact, a fully general invariance under change-of-variable transformations.
Most importantly, it implies that $A\mid X$ is a function only of
$\#_{S}\left(A\wedge X\right)$ and $\#_{S}(X)$ for any $S$ containing
all the symbols used in $A$ and $X$. R3 (invariance under addition
of irrelevant information) excludes Carnap's system, which comes pre-supplied
with dependencies between propositions rather than letting any dependencies
be specified in the premise. Adding R3 implies that $A\mid X$ is
a function of the \emph{ratio} of $\#_{S}\left(A\wedge X\right)$
to $\#_{S}(X)$. 

Finally, with R4 we made use of the underappreciated fact that CPL
already comes equipped with an inherent plausibility ordering on propositions,
for any given premise. The invariances of R1\textendash R3 ``stitch
together'' these partial orderings for distinct premises, and we
find that we have an order-preserving isomorphism $P$ between the
set of plausibilities $\mathbb{P}$ and the set of rational probabilities
$\mathbb{Q}_{01}$. The numeric value we find for $P\left(A\mid X\right)$
recreates the classical definition of probability, but in a sharper,
clearer form with the troublesome circularity excised, and as a \emph{theorem}
rather than as a definition. The ``possible cases'' are identified
as truth assignments satisfying the premise, and the meaning of ``equally
possible cases'' is simply that we have no additional information
that would expand the satisfying truth assignments by differing multiplicities.

We addressed two possible concerns about our result: that it seems
to allow only uniform probabilities, and that it yields probabilities
only for finite domains. We showed how non-uniform probabilities arise
via the introduction of latent variables, along with information in
the premise linking these latent variables to the observables. Following
Jaynes, we proposed that probabilities for infinite domains be obtained
via a well-defined and well-behaved limiting process, and demonstrated
how measure theory can automate the construction of such limiting
processes in at least some cases.

\appendix
\renewcommand*{\thesection}{\appendixname~\Alph{section}}

\section{Measure Theory\label{sec:Measure-Theory}}

\subsection{Some results from measure theory}

We assume the reader is already familiar with the basic concepts of
measure theory: an algebra, a $\sigma$-algebra, a measurable set,
the $\sigma$-algebra $\sigma(\mathcal{A})$ generated by an algebra
$\mathcal{A}$, a measurable function, a measure on an algebra or
$\sigma$-algebra, and a $\sigma$-finite measure. Billingsley \cite{billingsley}
and Tao \cite{tao-measure-theory} are good references. Here we highlight
some results we will use.

A measure on an algebra $\mathcal{A}$ can always be consistently
extended to a measure on $\sigma\left(\mathcal{A}\right)$ \cite[Theorem 11.3]{billingsley}:
\begin{thm}
\label{thm:extend-measure}If $\mu$ is a measure on an algebra $\mathcal{A}$
then $\mu$ extends to a measure on $\sigma\left(\mathcal{A}\right)$,
that is, there exists a measure $\mu'$ on $\sigma\left(\mathcal{A}\right)$
such that $\mu(A)=\mu'(A)$ for all $A\in\mathcal{A}$. If $\mu$
is $\sigma$-finite then $\mu'$ is unique, and is also $\sigma$-finite.
\end{thm}
In many cases of interest the elements of a $\sigma$-algebra can
be approximated arbitrarily closely by sets from the generating algebra
\cite[Theorem 11.4]{billingsley}:
\begin{thm}
\label{thm:approximate-measurable-set}If $\mathcal{A}$ is an algebra
on $\Omega$, $\mu$ is a $\sigma$-finite measure on $\sigma\left(\mathcal{A}\right)$,
and $B\in\sigma\left(\mathcal{A}\right)$ with $\mu\left(B\right)<\infty$,
then for every $\epsilon>0$ there exists some $A\in\mathcal{A}$
such that $\mu\left(A\triangle B\right)<\epsilon$.
\end{thm}
We use the following measure-related properties of set differences
$A\triangle B$, which we state without proof:
\begin{property}
\label{prop:diff-measure}For any measure $\mu$ on an algebra $\mathcal{A}$
and any $A,B\in\mathcal{A}$, 
\[
\left|\mu(A)-\mu(B)\right|\leq\mu\left(A\triangle B\right).
\]
\end{property}
{}
\begin{property}
\label{prop:intersect-diff}For any measure $\mu$ on an algebra $\mathcal{A}$
and any $A_{1},A_{2},X_{1},X_{2}\in\mathcal{A}$,
\[
\mu\left(\left(A_{1}\cap X_{1}\right)\triangle\left(A_{2}\cap X_{2}\right)\right)\leq\mu\left(A_{1}\triangle A_{2}\right)+\mu\left(X_{1}\triangle X_{2}\right).
\]
\end{property}

\subsection{Constructing the ``well-defined and well-behaved limiting process''}

To avoid confusion between propositional formulas and measurable sets,
in this section we will generally decorate the names of measurable
sets with a tilde ($\tilde{A}$, $\tilde{B}$, etc.) and leave the
names of propositional formulas undecorated ($A$, $B$, etc.)

The Borel $\sigma$-algebra and Borel measure for the Cantor set $\mathbb{B}^{\omega}$
are constructed as follows:
\begin{defn}
A \emph{cylinder set} is a subset of $\mathbb{B}^{\omega}$ of the
form $\mathrm{cyl}\left(n,C\right)\triangleq C\mathbb{B}^{\omega}$
for some $C\subseteq\mathbb{B}^{n}$. $\mathcal{C}_{0}$ is the collection
of all cylinder sets. This set is an algebra, and the Borel $\sigma$-algebra
for\emph{ $\mathbb{B}^{\omega}$} is $\mathcal{C}\triangleq\sigma\left(\mathcal{C}_{0}\right)$,
the $\sigma$-algebra generated by $\mathcal{C}_{0}$\emph{.}
\end{defn}
Cylinder sets are the basis of a topology on $\mathbb{B}^{\omega}$
in which the open sets are any finite or countable union of cylinder
sets, and this is why we call $\mathcal{C}$ the Borel $\sigma$-agebra
for $\mathbb{B}^{\omega}$.
\begin{defn}
The \emph{Borel measure for $\mathcal{C}$ is the measure} $\mu_{\mathcal{C}}$
such that 
\[
\mu_{\mathcal{C}}\left(\mathrm{cyl}\left(n,C\right)\right)=2^{-n}\left|C\right|
\]
for any $n\geq0$ and $C\subseteq\mathbb{B}^{n}$.
\end{defn}
The definition above is unambiguous because $\mathrm{cyl}(n,C)=\mathrm{cyl}(n+m,C')$
if and only if $C'=C\mathbb{B}^{m}$. Note that $\mu_{\mathcal{C}}$
is trivially $\sigma$-finite, since $\mathbb{B}^{\omega}$ itself
is a cylinder set and $\mu_{\mathcal{C}}\left(\mathbb{B}^{\omega}\right)$
is finite. By Theorem \ref{thm:extend-measure} $\mu_{\mathcal{C}}$
is uniquely defined once we define its value on cylinder sets.

Let us enumerate the elements of $\mathcal{S}$ as $s_{1},s_{2},\ldots$
and identify a sequence $w\in\mathbb{B}^{\omega}$ with the truth
assignment $\rho$ on $\mathcal{S}$ such that $\rho\left(s_{i}\right)=w_{i}$
for all $i$; then every propositional formula corresponds to a cylinder
set:
\begin{defn}
If $A$ is a propositional formula then $[A]$ is the set of $w\in\mathbb{B}^{\omega}$
that satisfy $A$ (considered as truth assignments.)
\end{defn}
Note that $[A]=\mathrm{cyl}(n,C)$ and $\mu_{\mathcal{C}}\left([A]\right)=2^{-n}\#_{S}(A)$,
where $\sigma\left\llbracket A\right\rrbracket \subseteq S=\left\{ s_{1},\ldots,s_{n}\right\} $
and $C$ is the set of $w\in\mathbb{B}^{n}$ that satisfy $A$. Likewise,
every cylinder set corresponds to a propositional formula:
\begin{lem}
\label{lem:cylinders-and-formulas}For any cylinder set $\tilde{A}$
there is a formula $A\in\Phi\left(\mathcal{S}\right)$ such that $\tilde{A}=[A]$.
\end{lem}
\begin{proof}
Let $\tilde{A}=\mathrm{cyl}\left(n,C\right)$ and define the propositional
formula $A$ as 
\begin{align*}
A & =\bigvee_{c\in C}A_{c}\\
A_{c} & =\bigwedge_{i=1}^{n}L_{i,c_{i}}\\
L_{i,0} & =\neg s_{i}\\
L_{i,1} & =s_{i}
\end{align*}
It is straightforward to see that $\left[A\right]=\tilde{A}$.
\end{proof}
We can define an analog to $P(A\mid X)$, but for measurable sets:
\begin{defn}
Let $\mathcal{A}$ be a $\sigma$-algebra and $\mu$ a measure on
$\mathcal{A}$. For any $\tilde{A},\tilde{X}\in\mathcal{A}$ with
$\mu\left(\tilde{X}\right)>0$, define

\[
\Pr\left(\tilde{A};\tilde{X},\mu\right)=\frac{\mu\left(\tilde{A}\cap\tilde{X}\right)}{\mu\left(\tilde{X}\right)}.
\]

\end{defn}
We then find that 
\[
P\left(A\mid X\right)=\frac{2^{-n}\#_{S}\left(A\wedge X\right)}{2^{-n}\#_{S}\left(X\right)}=\Pr\left([A];[X],\mu_{\mathcal{C}}\right)
\]
where we choose $n$ to be large enough that $\sigma\left\llbracket A,X\right\rrbracket \subseteq S=\left\{ s_{1},\ldots,s_{n}\right\} $.
We use this fact to show that Jaynes's ``well-defined and well-behaved
limiting process'' is guaranteed to exist for measurable sets:
\begin{thm}
\label{thm:well-defined-limit}Let $\tilde{A},\tilde{X}\in\mathcal{C}$,
with $\mu_{\mathcal{C}}\left(\tilde{X}\right)>0$. Then there exists
a sequence of formulas $A_{i}\in\Phi\left(\mathcal{S}\right)$ and
$X_{i}\in\Phi^{+}\left(\mathcal{S}\right)$ such that 
\begin{enumerate}
\item $\lim_{i\rightarrow\infty}\mu_{\mathcal{C}}\left(\left[A_{i}\right]\triangle\tilde{A}\right)=0$.
\item $\lim_{i\rightarrow\infty}\mu_{\mathcal{C}}\left(\left[X_{i}\right]\triangle\tilde{X}\right)=0$.
\item $\lim_{i\rightarrow\infty}P\left(A_{i}\mid X_{i}\right)=\Pr\left(\tilde{A};\tilde{X},\mu_{\mathcal{C}}\right)$.
\end{enumerate}
\end{thm}
\begin{proof}
Let $\epsilon_{i}$, $i\geq1$ be any decreasing sequence of positive
numbers whose limit is 0, with $\epsilon_{1}<\mu_{\mathcal{C}}\left(\tilde{X}\right)$.
Using Theorem \ref{thm:approximate-measurable-set} and Lemma \ref{lem:cylinders-and-formulas}
we can define 
\begin{align*}
A_{i} & =\mbox{some }A\in\Phi\left(\mathcal{S}\right)\mbox{ such that }\mu_{\mathcal{C}}\left([A]\triangle\tilde{A}\right)<\epsilon_{i}\\
X_{i} & =\mbox{some }X\in\Phi\left(\mathcal{S}\right)\mbox{ such that }\mu_{\mathcal{C}}\left([X]\triangle\tilde{X}\right)<\epsilon_{i}
\end{align*}
1 and 2 in the theorem statement follow directly from these definitions.
From Property \ref{prop:diff-measure} we have 
\[
\left|\mu_{\mathcal{C}}\left(\left[X_{i}\right]\right)-\mu_{\mathcal{C}}\left(\tilde{X}\right)\right|\leq\mu_{\mathcal{C}}\left([X]\triangle\tilde{X}\right)<\epsilon_{i}<\mu_{\mathcal{C}}\left(\tilde{X}\right)
\]
and so $\mu_{\mathcal{C}}\left(\left[X_{i}\right]\right)>0$, i.e.,
$X_{i}$ is satisfiable. Property \ref{prop:diff-measure} also gives
us 
\[
\lim_{i\rightarrow\infty}\mu_{\mathcal{C}}\left(\left[X_{i}\right]\right)=\mu_{\mathcal{C}}\left(\tilde{X}\right).
\]
Property \ref{prop:intersect-diff} gives us 
\[
\mu_{\mathcal{C}}\left(\left[A_{i}\wedge X_{i}\right]\triangle\left(\tilde{A}\cap\tilde{X}\right)\right)<2\epsilon_{i}
\]
and then Property \ref{prop:diff-measure} yields 
\[
\lim_{i\rightarrow\infty}\mu_{\mathcal{C}}\left(\left[A_{i}\wedge X_{i}\right]\right)=\mu_{\mathcal{C}}\left(\tilde{A}\cap\tilde{X}\right).
\]
Finally we have 
\[
\lim_{i\rightarrow\infty}P\left(A_{i}\mid X_{i}\right)=\lim_{i\rightarrow\infty}\frac{\mu_{\mathcal{C}}\left(\left[A_{i}\wedge X_{i}\right]\right)}{\mu_{\mathcal{C}}\left(\left[X_{i}\right]\right)}=\frac{\mu_{\mathcal{C}}\left(\tilde{A}\cap\tilde{X}\right)}{\mu_{\mathcal{C}}\left(\tilde{X}\right)}=\Pr\left(\tilde{A};\tilde{X},\mu_{\mathcal{C}}\right).
\]
\end{proof}

\end{document}